\PassOptionsToPackage{hyphens}{url}
\documentclass[lettersize,journal]{IEEEtran}
\usepackage{amsmath,amsfonts}
\usepackage{algorithmic}
\usepackage{array}
\usepackage[caption=false,font=normalsize]{subfig}
\usepackage{textcomp}
\usepackage{stfloats}
\usepackage{url}
\usepackage{verbatim}
\usepackage{graphicx}
\def\BibTeX{{\rm B\kern-.05em{\sc i\kern-.025em b}\kern-.08em
    T\kern-.1667em\lower.7ex\hbox{E}\kern-.125emX}}
\usepackage{balance}

\usepackage{colortbl}
\usepackage{tabulary}
\definecolor{Gray}{gray}{0.9}
\usepackage{amsmath}
\usepackage{subfloat}
\usepackage[utf8]{inputenc} 
\usepackage[T1]{fontenc}    
\usepackage{hyperref}       
\usepackage{url}            
\usepackage{booktabs}       
\usepackage{amsfonts}       
\usepackage{nicefrac}       
\usepackage{microtype}      
\usepackage{xcolor}         
\usepackage{anyfontsize}
\usepackage{multirow}
\usepackage{wrapfig}
\usepackage{hyperref}
\usepackage{url}
\usepackage[numbers]{natbib}
\usepackage{enumitem}
\usepackage{amsmath,amsthm,amssymb}
\usepackage{algorithm}
\usepackage{algorithmic}
\usepackage{graphicx}
\usepackage{colortbl}
\usepackage{dsfont}
\usepackage{float} 
\usepackage{caption}
\usepackage{booktabs}
\usepackage{makecell}
\usepackage{diagbox}
\usepackage{fancybox}
\usepackage{bm}
\usepackage{color}

\newtheorem{lemma}{Lemma}
\def\red#1{\textcolor{black}{#1}}

\newcommand{\tabincell}[2]{\begin{tabular}{@{}c#1@{}}#2\end{tabular}}

\begin{document}
\title{Backdoor Attack with Sparse and Invisible Trigger}
\author{Yinghua Gao$^{\ast}$,
        Yiming Li$^{\ast}$,
        Xueluan Gong,
        Zhifeng Li,
        Shu-Tao Xia,
        Qian Wang, \IEEEmembership{Fellow,~IEEE}
\thanks{The first two authors contributed equally to this work.}
\thanks{Yinghua Gao is with the Tsinghua Shenzhen International Graduate School, Tsinghua University, Shenzhen, 518055, China (e-mail: yhgao18@gmail.com).}
\thanks{Yiming Li is with The State Key Laboratory of Blockchain and Data Security, Zhejiang University, Hangzhou, 311200, China and also with College of Computing and Data Science, Nanyang Technological University, Singapore, 639798 (e-mail: liyiming.tech@gmail.com).}
\thanks{Xueluan Gong is with the School of Computer Science, Wuhan University, China (e-mail: xueluangong@whu.edu.cn).}
\thanks{Zhifeng Li is with Tencent Data Platform, Shenzhen, 518057, China (email: michaelzfli@tencent.com).}
\thanks{Shu-Tao Xia is with the Tsinghua Shenzhen International Graduate School, Tsinghua University, Shenzhen, 518055, China, and also with the Research Center of Artificial Intelligence, Peng Cheng Laboratory, Shenzhen, 518000, China (e-mail: xiast@sz.tsinghua.edu.cn).}
\thanks{Qian Wang is with the School of Cyber Science and Engineering, Wuhan University, China (e-mail: qianwang@whu.edu.cn). }
\thanks{Corresponding Author: Yiming Li (e-mail: \href{mailto:liyiming.tech@gmail.com}{liyiming.tech@gmail.com}).}
}

\markboth{IEEE Transactions on Information Forensics and Security}%
{IEEE Transactions on Information Forensics and Security}


\maketitle

\begin{abstract}
Deep neural networks (DNNs) are vulnerable to backdoor attacks, where the adversary manipulates a small portion of training data such that the victim model predicts normally on the benign samples but classifies the triggered samples as the target class. The backdoor attack is an emerging yet threatening training-phase threat, leading to serious risks in DNN-based applications. In this paper, we revisit the trigger patterns of existing backdoor attacks. We reveal that they are either visible or not sparse and therefore are not stealthy enough. More importantly, it is not feasible to simply combine existing methods to design an effective sparse and invisible backdoor attack. To address this problem, we formulate the trigger generation as a bi-level optimization problem with sparsity and invisibility constraints and propose an effective method to solve it. The proposed method is dubbed sparse and invisible backdoor attack (SIBA). We conduct extensive experiments on benchmark datasets under different settings, which verify the effectiveness of our attack and its resistance to existing backdoor defenses. The codes for reproducing main experiments are available at \url{https://github.com/YinghuaGao/SIBA}.

\end{abstract}


\section{Introduction}
\IEEEPARstart{D}{eep} neural networks (DNNs) have demonstrated their effectiveness and been widely deployed in many mission-critical applications ($e.g.$, facial recognition \cite{liu2006spatio,li2016mutual,luo2022memory}). Currently, training a well-performed model generally requires a large amount of data and computational consumption that are costly. Accordingly, \red{researchers and developers usually choose to exploit third-party training resources ($e.g.$, open-sourced data, cloud computing platforms, and pre-trained models) to alleviate training burdens in practice}.

\textcolor{black}{However, the convenience of using third-party training resources may also lead to new security threats. One of the most emerging yet severe threats is called the backdoor attack \cite{gu2019badnets,wang2021stop,Gong2021,li2022backdoor,dong2023one,gong2023kaleidoscope}, where the adversaries intend to implant a hidden backdoor to victim models during the training process. The backdoored models behave normally on predicting benign samples whereas their predictions will be maliciously changed to the adversary-specified target class whenever the backdoor is activated by the adversary-specified trigger pattern. {Backdoor attacks severely reduce the model’s reliability and attract tremendous attention from various domains. A recent industry report \cite{kumar2020adversarial} demonstrated that companies are concerned about backdoor attacks and rank them as the fourth among the most popular security threats. Besides, government agencies also placed a high priority on backdoor research. For instance, U.S. intelligence community \cite{RN16} launched a new funding program to defend against backdoor attacks and some other threats.}}


The design of trigger patterns is one of the most important factors in backdoor attacks. Currently, there are many different types of trigger patterns \cite{nguyen2021wanet,cheng2021deep,gong2023kaleidoscope}. Arguably, patch-based triggers \cite{gu2019badnets,li2022few,qi2023revisiting} and additive triggers \cite{barni2019new,zhao2020clean,li2021invisible} are the most classical and widely adopted ones among all trigger patterns\footnote{\textcolor{black}{The trigger patterns vary depending on the specific tasks. In this work, we focus on image classification tasks.} }. \textcolor{black}{Specifically, patch-based triggers are extremely sparse since they only modify a small number of pixels. Accordingly, they can be implemented as stickers, making the attacks more feasible in the physical world. However, they are visible to human eyes and can not evade inspection. On the other hand, additive triggers are usually invisible but modify almost all pixels, restricting the feasibility in real scenarios.} \textcolor{black}{Overall, it is crucial for a backdoor trigger to maintain both sparseness and imperceptibility since the sparseness benefits the practical implementation while imperceptibility helps to evade immediate detection and post hoc forensic analysis. However, previous works \cite{gu2019badnets, barni2019new, zhao2020clean, li2021invisible, li2022few, qi2023revisiting} only satisfy at most one requirement.} An intriguing question arises: \emph{Is it possible to design a trigger pattern for effective backdoor attacks that is both sparse and invisible?}

The answer to the aforementioned question is positive. \textcolor{black}{In general, the most straightforward method is to introduce a (random) sparse mask to the patterns of existing invisible backdoor attacks during their generation process. However, as we will show in the experiments, this method is not effective in many cases, especially when the image size is relatively large. Its failure is mostly because the position of trigger areas is also critical for their success, especially when the trigger size and perturbation strength are limited. Based on these understandings, in this paper, we propose to design sparse and invisible backdoor attacks by optimizing trigger areas and patterns simultaneously. Specifically, we formulate it as a bi-level optimization problem with sparsity and invisibility constraints. The upper-level problem is to minimize the loss on poisoned samples via optimizing the trigger, while the lower-level one is to minimize the loss on all training samples via optimizing the model weights. In particular, this optimization problem is difficult to solve directly due to the high non-convexity and poor convergence of gradient-based methods. To alleviate these problems, we exploit a pre-trained surrogate model to reduce the complexity of lower-level optimization and derive an alternate projected method to satisfy the $L_\infty$ and $L_0$ constraints.} We conduct extensive experiments on benchmark datasets to demonstrate the effectiveness of the proposed method and its feasibility under different settings. Besides, as we will show in the experiments, the generated sparse trigger pattern contains semantic information about the target class. \textcolor{black}{It indicates our attack may serve as the potential path toward explaining DNNs, which is an unexpected finding. }

Our main contributions can be summarized as follows.

\begin{itemize}
    \item We reveal the potential limitations of current backdoor attacks that the triggers could not satisfy the sparsity and invisibility constraints simultaneously. 
    \item We formulate the sparse and invisible backdoor attack as a bi-level optimization and develop an effective and efficient method to solve it.
    \item We conduct extensive experiments on benchmark datasets, verifying the effectiveness of our attack and its resistance to potential backdoor defenses.
\end{itemize}

\vspace{0.3em}
The rest of this paper is organized as follows: We briefly review some related works in Section \ref{sec:relatedwork} and introduce the proposed method in Section \ref{sec:method}. We present the experimental results and analysis in Section \ref{sec:exp} and conclude the whole paper in Section \ref{sec:conclusion} at the end.

\section{Related Work}
\label{sec:relatedwork}

\subsection{Backdoor Attack}
Backdoor attacks aim to inject backdoor behaviors to the victim model such that the attacked model behaves normally on benign test samples but predicts the target class whenever the trigger is attached to the test samples. Trigger design is the core of backdoor attacks and a large corpus of works are devoted to proposing better triggers. In general, existing trigger designs can be roughly divided into two main categories, including trigger patches and additive perturbations, as follows. 

\vspace{0.3em}
\noindent \textbf{Backdoor Attacks with Patch-based Triggers.} BadNets \cite{gu2019badnets} is the first backdoor attack designed with the patch-based trigger. In general, the adversaries randomly select a few benign samples from the original dataset and `stamp' the pre-defined black-and-white trigger patch to those images and reassign their labels to a target class. These modified samples (dubbed `poisoned samples') associated with remaining benign samples will be released as the poisoned dataset to victim users. After that, Chen \textit{et al.} \cite{chen2017targeted} proposed a blended injection strategy to make the poisoned samples hard to be noticed by humans by introducing trigger transparency. Recently, Li \textit{et al.} \cite{li2021backdoor} discussed how to design physical backdoor attacks that are still effective when the poisoned samples are directly captured by digital devices ($e.g.$, camera) in real-world scenarios where trigger angles and positions could be changed. Most recently, Li \textit{et al.} \cite{li2022untargeted} adopted patch-based triggers to design the first untargeted backdoor attack. In general, patch-based triggers are the most classical and even the default setting for new tasks \cite{li2022few,luo2023untargeted,chen2023trojdiff}. \textcolor{black}{The distinct convenience is that the triggers exhibit remarkable sparsity, and therefore, poisoned images can be obtained by attaching the stickers, facilitating their threats in physical scenarios \cite{gu2019badnets,li2022backdoor}. The adversaries usually limited the sparsity ($i.e.$, trigger size) of patch-based triggers for stealthiness. However, although a high trigger transparency may be used, the perturbations are still visible to a large extent since the trigger patterns are significantly different from the replaced portions in the poisoned images.} 

\vspace{0.3em}
\noindent \textbf{Backdoor Attacks based on Additive Perturbations.} Recently, using additive perturbations as trigger patterns becomes popular in backdoor attacks. For example, Zhao \textit{et al.} \cite{zhao2020clean} adopted the target universal adversarial perturbation as the trigger to design an effective clean-label backdoor attacks where the target label is consistent with the ground-truth label of the poisoned samples; Nguyen \textit{et al.} \cite{nguyen2021wanet} generated trigger patterns by image warping; Li \textit{et al.} \cite{li2021invisible} adopted a pre-trained attribute encoder to generate additive trigger patterns, inspired by deep image steganography. \textcolor{black}{Compared to patch-based backdoor attacks, these methods are more controllable regarding trigger stealthiness since the adversaries can easily ensure invisibility by limiting the maximum perturbation size of trigger patterns. However, to the best of our knowledge, almost all existing methods need to modify the whole image for poisoning, restricting the feasibility in the physical world}. How to design attacks with invisible and sparse trigger patterns remains an important open question and worth further explorations.

Recently, there were also a few works exploiting and designing backdoor attacks for positive purposes \citep{li2023black,guo2023domain,lin2021you,ya2024towards}, which are out of the scope of this paper.


\subsection{Backdoor Defense}
Currently, there are some methods (dubbed `backdoor defenses') to reduce the backdoor threats. In general, existing defenses can be divided into \textcolor{black}{five} main categories, including \textbf{(1)} the detection of poisoned training samples, \textbf{(2)} poison suppression, \textcolor{black}{\textbf{(3)} backdoor removal, \textbf{(4)} the detection of poisoned testing samples, and \textbf{(5)} the detection of attacked models}. 

Specifically, the first type of defense intends to filter out poisoned training samples \cite{chen2018detecting, wang2019neural,hayase2021spectre}. They are either based on the representation differences between poisoned and benign samples at intermediate layers of the model or directly attempt to reverse the trigger pattern; Poison suppression \cite{li2021anti,huang2022backdoor,tang2023setting} depresses the effectiveness of poisoned samples during the training process to prevent the creation of hidden backdoors; Backdoor removal aims to remove the hidden backdoors in give pre-trained models \cite{liu2018fine, zeng2022adversarial,li2024nearest}; The forth type of defense \cite{chou2020sentinet,gao2021design,guo2023scale} targets the detection of poisoned testing samples, \textcolor{black}{and the last type of defense determines whether a given model is backdoored based on some model properties \cite{wang2019neural,xu2021detecting, xu2024towards}}. 


\subsection{Sparse Optimization}

Sparse optimization requires the most elements of the variable to be zero, which usually brings unanticipated benefits such as interpretability \cite{fan2020sparse, zhu2021sparse} or generalization \cite{zhou2022sparse, li2022sparse}. It has been extensively studied in various applications such as basis pursuit denoising \cite{fevotte2006sparse}, compressed sensing \cite{candes2006robust}, source coding \cite{donoho1998data}, model pruning \cite{li2019compressing, yang2020automatic} and adversarial attacks \cite{croce2019sparse, fan2020sparse, zhu2021sparse}. The popular approaches to solve sparse optimization include relaxed approximation method \cite{candes2005decoding,xu2012l_, zhang2010analysis} which penalized the original objective with regularizers such as $L_1$-norm, top-$k$ norm and Schatten $L_p$ norm, and proximal gradient method \cite{beck2013sparsity, jain2014iterative, lu2014iterative} which exploited the proximal operator that can be evaluated analytically. In this paper, we provide a unified formulation of sparse and invisible backdoor attacks and derive a projected method to practically optimize the trigger. We notice that similar techniques were used in recent adversarial attacks \cite{croce2019sparse, dong2020greedyfool}. However, our method differs from them in that \textbf{(1)} Existing methods were test-time attacks while our method is training-time attack. \textbf{(2)} These papers focused on data-wise optimization while our method studies on group-wise optimization, which is a much more challenging task.


\section{The Proposed Method}
\label{sec:method}
\subsection{Preliminaries}

\noindent {\bf{Threat Model.}} In this paper, we study the image classification task where the model outputs a $C$-class probability vector: $f_{\bm \theta}:\mathbb{R}^d\rightarrow [0, 1]^C$.  Given a training set $\mathcal{D}=\left\{{\bm x}_i, y_i\right\}_{i=1}^N$, $\bm x\in \mathbb{R}^d$, $y\in \left[C\right]=\left\{1, 2, \cdots, C\right\}$, the parameters of classifier are optimized by the empirical risk minimization: 
\[\min_{\bm \theta} \sum_{({\bm x},y)\in \mathcal{D}} \mathcal{L}(f_{\bm \theta}({\bm x}), y),\]
where $\mathcal{L}$ represents the loss function ($e.g.$, the cross entropy loss). 
\textcolor{black}{Following the analysis framework in \cite{biggio2018wild}, we demonstrate our threat model from four perspectives: adversary's goal, knowledge, capability, and strategy, as follows.}


\vspace{0.3em}
\noindent \textcolor{black}{\textit{Adversary's Goal.} In general, the adversaries have three main goals, including the \emph{utility}, \emph{effectiveness}, and \emph{stealthiness}. Specifically, the utility requires that the attacked model $f_{\bm \theta}$ achieves high accuracy on benign test samples. Otherwise, the model would not be adopted and therefore no backdoor could be implanted; The effectiveness desires that the attacked model can achieve high attack success rates whenever trigger patterns appear; The stealthiness requires that the dataset modification should be unnoticeable to victim dataset users. For example, the trigger patterns should be invisible and sparse, while the poisoning rate should be small.} 

\vspace{0.3em}
\noindent \textcolor{black}{\textit{Adversary's Knowledge.} We assume the adversary has access to (a few) training data but neither the learning algorithm nor the objective function during the training. Our settings align with previous works on backdoor attacks \cite{gu2019badnets, barni2019new, li2021invisible} and resemble the limited-knowledge gray-box attacks proposed in \cite{biggio2018wild}.}

\vspace{0.3em}
\noindent \textcolor{black}{\textit{Adversary's Capability.} The adversary is only allowed to modify a small subset $\mathcal{D}_s$ of the original training set $\mathcal{D}$ ($i.e.$, $\mathcal{D}_s\subset \mathcal{D}$ and $|\mathcal{D}_s|\ll |\mathcal{D}|$) by attaching the trigger $\bm t$ and relabelling them as the target class $y_T$. The victim model $f_{\bm \theta}$ is trained on the modified dataset $\mathcal{D}^{\prime}$, which is composed of a benign dataset $\mathcal{D}_c=\mathcal{D}\setminus \mathcal{D}_s$ and a poisoned dataset $\mathcal{D}_p=\left\{({\bm x}+{\bm t}, y_T)|({\bm x},y)\in \mathcal{D}_s\right\}$. In particular, the ratio $|\mathcal{D}_p|/|\mathcal{D}|$ is called as the \emph{poisoning rate}.}

\vspace{0.3em}
\noindent \textcolor{black}{\textit{Attack Strategy.} We formulate the proposed attack as a bi-level optimization. Its details are in the next subsection.}


\vspace{0.3em}
\noindent \textcolor{black}{{\bf Characterization of Sparsity and Invisibility.} Given a $d$-dimensional vector ${\bm x}$, $L_0$-norm is defined as $\Vert{\bm x}\Vert_0=\sum_{i=1}^d \mathbb{I}({\bm x}_i\neq 0)$ where $\mathbb{I}(\cdot)$ is the indicator function. In general, $L_0$-norm reflects the maximum number of pixels allowed for modification and is widely used to measure sparsity in recent works \cite{li2019compressing,fan2020sparse}. $L_\infty$-norm is defined as $\Vert{\bm x}\Vert_\infty=\max_{1\leq i\leq d} \vert {\bm x}_i \vert$, which represents the maximum absolute value of the elements. It is widely adopted to reflect invisibility, especially in adversarial attacks \cite{xie2017adversarial,dong2019efficient}.}


\subsection{Sparse and Invisible Backdoor Attack (SIBA)}
\label{sec:SIBA_method}
As we mentioned in the previous section, we need to optimize the trigger sparsity and visibility simultaneously to ensure better stealthiness. In this section, we introduce the formulation and optimization of our sparse and invisible backdoor attack.

\vspace{0.3em}
\noindent {\bf{Problem Formulation.}} The objective of SIBA could be formulated as a bi-level optimization problem since the effectiveness of the trigger pattern is related to a trained model whose optimization is also influenced by poisoned samples, as follows:
\begin{equation}
\begin{aligned}
    \min_{{\bm t}} \sum_{({\bm x},y)\in \mathcal{D}_v}& \mathcal{L}(f_{\bm w}({\bm x}+{\bm t}), y_T) \\
    s.t. \ {\bm w}=\arg\min_{\bm \theta} &\sum_{({\bm x},y)\in \mathcal{D}_c\bigcup\mathcal{D}_p}\mathcal{L}(f_{\bm \theta}({\bm x}), y), \\ \underbrace{\lVert \bm t \rVert_0 \leq k}_{\textbf{sparsity}}&, \ \underbrace{\lVert \bm t \rVert_\infty \leq \epsilon}_{\textbf{invisibility}},
\end{aligned}
\label{eq:siba_form}
\end{equation} 
where $\mathcal{D}_v$ denotes the validation set acquired by the adversary. The upper-level optimization aims to ensure the effectiveness of the trigger, that is, the trained model $f_{\bm w}$ would classify the samples attached with the triggers as the target class. The lower-level optimization represents the training process of the victim model. Besides, we add $L_0$ and $L_\infty$ constraints to confirm the trigger's sparsity and invisibility.

\begin{figure*}
    \centering
    \vspace{-1em}
    \includegraphics[width=1.0\linewidth]{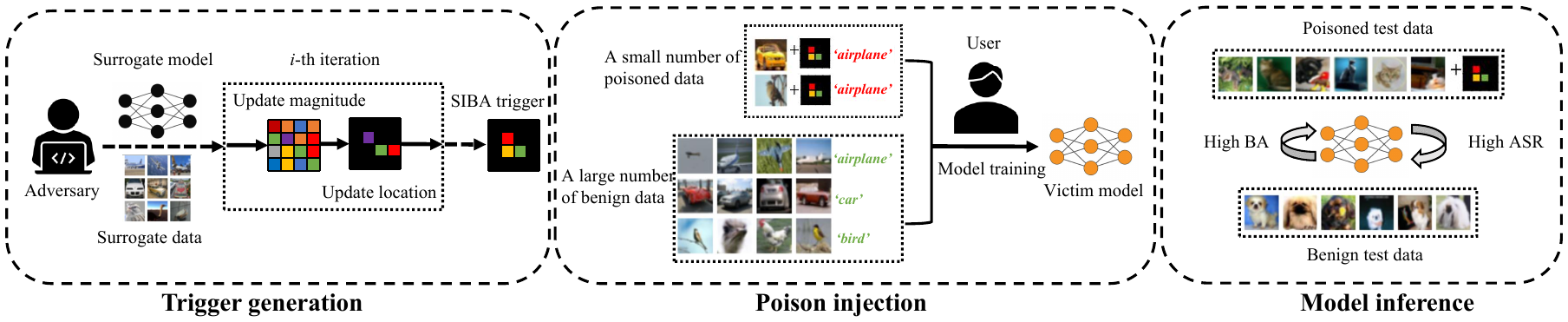}
    \caption{The main pipeline of our sparse and invisible backdoor attack (SIBA). In the first step, the adversaries will generate a sparse and invisible trigger pattern. In the second step, the victim users will train their model on both benign and poisoned samples released by the adversary. In the third step, the attacked model can correctly predict clean test samples whereas the adversaries can maliciously change its predictions to the target class (`airplane' in this example). }
    \label{fig:siba_pipe}
\end{figure*}

\begin{algorithm}[!t] 
	\caption{The optimization process of our SIBA.} 
	\label{alg:siba} 
	\begin{algorithmic}[1] 
		\REQUIRE ~~\\ 
		Batch size $M$, step size $\alpha$, number of training iterations $T$, update step $K$, pre-trained model $f_{\bm b}$, the number of maximum perturbed pixels $k$, $L_\infty$ constraint $\epsilon$, target label $y_T$, mask vector $\bm m$, validation set $\mathcal{D}_v$. \\
		Initialize ${\bm t}\leftarrow {\bm 0}, j\leftarrow 0, i\leftarrow 0, {\bm m}\leftarrow {\bm 0}$. $\mathcal{B}_{\epsilon}=\left\{ {\bm t} \ \lvert \  \lVert \bm t \rVert_\infty \leq \epsilon\right\}$.\\
        \WHILE{$j<T$}
        \STATE Sample mini-batch of $M$ samples $\left\{{\bm x}_i, y_i\right\}_{i=1}^M$ from $\mathcal{D}_v$
        \IF{$i\bmod K=0$}
        \STATE Select the top $k$ large dimension $d_1, \cdots, d_k$ from $\lvert\nabla_{\bm t} \sum_{i=1}^M \mathcal{L}(f_{\bm b}({\bm x}_i+{\bm t}), y_T)\rvert$
        \STATE ${\bm m}_{d_1,\cdots,d_k}\leftarrow 1$
        \STATE ${\bm m}_{[d]\backslash d_1,\cdots,d_k}\leftarrow 0$
        \ENDIF
        \STATE ${\bm t}\leftarrow \Pi_{\mathcal{B}_{\epsilon}}\left({\bm t}-\alpha\cdot\epsilon\cdot\text{sign}(\nabla_{\bm t} \sum_{i=1}^M \mathcal{L}(f_{\bm b}({\bm x}_i+{\bm t}), y_T))\right)$
        \STATE ${\bm t}\leftarrow {\bm m}\cdot {\bm t}$
        \STATE $j\leftarrow j+1$
        \STATE $i\leftarrow i+1$
        \ENDWHILE
		\RETURN The optimized trigger $\bm t$.
	\end{algorithmic}
\end{algorithm}

\vspace{0.3em}
\noindent {\bf{Surrogate Optimization Problem.}} Due to the high non-convexity of the optimization Problem \ref{eq:siba_form}, we need to seek a feasible solution. In particular, the optimization of ${\bm \theta}$ in the lower-level optimization requires full model training, which is time-consuming. To alleviate the computational burden and optimization difficulties, we exploit a pre-trained benign model $f_{\bm b}$ to replace $f_{\bm w}$ in the upper-level optimization. Instead of solving the Problem \ref{eq:siba_form} directly, we turn it into the following surrogate optimization problem: 
\begin{equation}
\begin{aligned}
    \min_{\bm t} \sum_{({\bm x},y)\in \mathcal{D}_v}& \mathcal{L}(f_{\bm b}({ \bm x}+{\bm t}), y_T) \\
    s.t. \ \underbrace{\lVert \bm t \rVert_0 \leq k}_{\textbf{sparsity}}&, \ \underbrace{\lVert \bm t \rVert_\infty \leq \epsilon}_{\textbf{invisibility}}.
\end{aligned}
\label{eq:siba_approx}
\end{equation} 
In this way, we only need to train the surrogate model $f_{\bm b}$ once and avoid frequent updates. We will demonstrate the feasibility and rationality of the surrogate optimization by showing the attack transferability in Section \ref{sec:exp}.

\vspace{0.3em}
\noindent {\bf{Practical Optimization.}} Let $h({\bm t})=\sum_{({\bm x},y)\in \mathcal{D}_v} \mathcal{L}(f_{\bm b}({\bm x}+{\bm t}), y_T)$, we investigate the update of $h({\bm t})$ under the $L_\infty$ and $L_0$ constraints sequentially. Firstly, to satisfy the $L_\infty$ constraint in Problem \ref{eq:siba_approx}, we utilize the projected gradient method, which has been extensively explored in adversarial training \cite{madry2017towards}. The $i$-th update formula is shown as follows:
\begin{equation}
    {\bm v}_{i} = \Pi_{\mathcal{B}_{\epsilon}}\left({\bm t}_i - \alpha\cdot\epsilon\cdot \text{sign}(\nabla_{\bm t} h({\bm t}_i))\right),
\label{eq:fgsm}
\end{equation}
where $\mathcal{B}_\epsilon=\left\{ {\bm t} \ \lvert \  \lVert \bm t \rVert_\infty \leq \epsilon\right\}$, $\alpha$ is the step size. Next, we attempt to project ${\bm v}_{i}$ into the $L_0$ box: $\mathcal{B}=\left\{{\bm t}\ \lvert \ \lVert {\bm t} \rVert_0\leq k, \ {\bm t}_{j} = {\bm v}_{i, j} \ \text{or} \ 0, \ j=1,2\cdots,d\right\}$, which means we must select at most $k$ element of ${\bm v}_{i}$ and set the other elements to zero. We denote ${\bm s}_i={\bm t}_i - \alpha\cdot\nabla_{\bm t} h({\bm t}_i)$ and require the projected ${\bm t}_{i+1}$ as close to ${\bm s}_i$ as possible in the terms of square loss, as follows: 
\begin{equation}
\label{eq:square}
    {\bm t}_{i+1} = \arg\min_{{\bm u}\in \mathcal{B}} \ \lVert {\bm s}_i- {\bm u}  \rVert_2^2.  
\end{equation}

To solve Problem \ref{eq:square}, we have the following Lemma.
\begin{lemma}
\label{lem}
Assuming $\alpha=0$ in Equation \ref{eq:fgsm} and the initial value of ${\bm t}_i$ is 0, Problem \ref{eq:square} has the analytical solution as follows:
\begin{equation}
    {\bm t}_{i+1, j} = \left\{
\begin{aligned}
 &{\bm v}_{i, j} &\text{if} \ j \in C^{\prime}\\
&0  &\text{if} \  j \notin C^{\prime}
\end{aligned},
\right.
\label{eq:l0_proj}
\end{equation}
where $C^{\prime}$ represents the subscript group which has the largest $k$ element of $\lvert \nabla_{\bm t} h({\bm t}_i) \rvert$.
\end{lemma}

\begin{figure*}
    \centering
    \subfloat[Clean]{
     \includegraphics[width=0.1\linewidth]{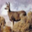}}
    \subfloat[BadNets]{
     \includegraphics[width=0.1\linewidth]{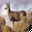}}
    \subfloat[Blended]{
     \includegraphics[width=0.1\linewidth]{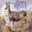}}
    \subfloat[TUAP]{
     \includegraphics[width=0.1\linewidth]{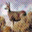}}
    \subfloat[WaNet]{
     \includegraphics[width=0.1\linewidth]{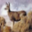}}
    \subfloat[ISSBA]{
     \includegraphics[width=0.1\linewidth]{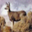}}
    \subfloat[Random]{
     \includegraphics[width=0.1\linewidth]{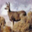}}
    \subfloat[Sparse]{
     \includegraphics[width=0.1\linewidth]{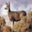}}
    \subfloat[SIBA]{
     \includegraphics[width=0.1\linewidth]{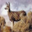}}
    
    \subfloat[Clean]{
     \includegraphics[width=0.1\linewidth]{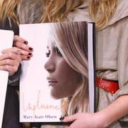}}
    \subfloat[BadNets]{
     \includegraphics[width=0.1\linewidth]{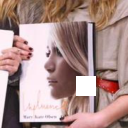}}
    \subfloat[Blended]{
     \includegraphics[width=0.1\linewidth]{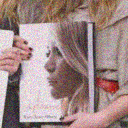}}
    \subfloat[TUAP]{
     \includegraphics[width=0.1\linewidth]{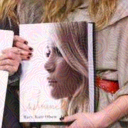}}
    \subfloat[WaNet]{
     \includegraphics[width=0.1\linewidth]{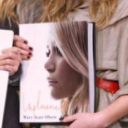}}
    \subfloat[ISSBA]{
     \includegraphics[width=0.1\linewidth]{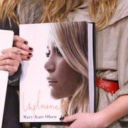}}
    \subfloat[Random]{
     \includegraphics[width=0.1\linewidth]{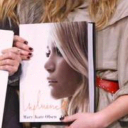}}
    \subfloat[Sparse]{
     \includegraphics[width=0.1\linewidth]{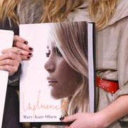}}
    \subfloat[SIBA]{
     \includegraphics[width=0.1\linewidth]{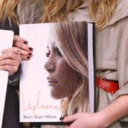}}
    \caption{The examples of poisoned samples with different backdoor attacks on CIFAR-10 and VGGFace2 datasets. \textbf{First Row}: poisoned samples on the CIFAR-10 dataset. \textbf{Second Row}: poisoned samples on the VGGFace2 dataset.}
    \label{fig:vary_cifar_backdoor}
\end{figure*}

The proof of Lemma \ref{lem} is in the appendix. Based on the above analysis, each iteration consists of Step \ref{eq:fgsm} and Step \ref{eq:l0_proj}. We exploit a mask $\bm m\in \left\{0,1\right\}^d$ to perform Step \ref{eq:l0_proj} and update $\bm m$ after multiple iterations of Step \ref{eq:fgsm} to stabilize the optimization. The detailed optimization process is shown in Algorithm \ref{alg:siba}.

\vspace{0.3em}
\noindent {\bf{Poison Injection.}} We attach the optimized trigger ${\bm t}$ to a small portion of training data and relabel them as the target class. The modified training set that consists of the triggered set and the untouched set is released to the users for training the victim models. The training configurations of the victim models are determined by the users while the adversary can not intervene.

\vspace{0.3em}
\noindent {\bf{Model Inference.}} During the inference phase, the adversary validates the effectiveness of the backdoor attack with two types of data: benign test data and poisoned test data. The victim model is expected to predict correct class for the benign test data and the target class for the poisoned test data. The whole pipeline of our SIBA is shown in Figure \ref{fig:siba_pipe}.

\section{Experiments}
\label{sec:exp}

\subsection{Experimental Setup}
\label{sec:main_exp}

\noindent{\bf Datasets.} In this paper, we consider CIFAR-10 and VGGFace2 datasets in our experiments. CIFAR-10 dataset consists of 10 classes and each class includes 5,000 training images and 1,000 test images. The size of each image is $32\times 32\times 3$. For VGGface2 dataset, we construct a 20-class subset from the original set for training efficiency. Each class includes 400 training images and 100 test images. The size of each image is $128\times 128\times 3$. Both datasets are commonly used in recent backdoor-related research \cite{nguyen2021wanet, li2022untargeted, doan2020februus,dumford2020backdooring,wenger2021backdoor}. 

\vspace{0.3em}
\noindent{\bf Baseline Selection.} We compare our SIBA method with six classical and representative backdoor attacks, including \textbf{(1)} BadNets \cite{gu2019badnets}, \textbf{(2)} backdoor attack with the blended strategy (dubbed as `Blended') \cite{chen2017targeted}, \textbf{(3)} TUAP \cite{zhao2020clean}, \textbf{(4)} WaNet \cite{nguyen2021wanet}, \textbf{(5)} ISSBA \cite{li2021invisible}, and \textbf{(6)} UBW-P \cite{li2022untargeted}. Among the aforementioned baselines, the triggers of BadNets and UBW-P are sparse (small $L_0$ constraint) but visible (large $L_\infty$ constraint), while those of others (Blended, TUAP, WaNet, ISSBA) are invisible (small $L_\infty$ constraint) but dense (large $L_0$ constraint). We also provide the results of two straightforward (yet ineffective) sparse and invisible backdoor attacks, including \textbf{(1)} using a random noise as the trigger that is restricted to have the same $L_\infty$ and $L_0$ constraint with our SIBA (dubbed as `Random') and \textbf{(2)} the improved version of Random where its trigger magnitude is optimized using Line 8 in Algorithm \ref{alg:siba}. 


\vspace{0.3em}
\noindent{\bf Attack Settings.} 
For all attacks, we set the poisoning rate as $1\%$ and choose the class `0' as the target class. Specifically, for the settings of BadNets, the trigger is a $3\times 3$ checkerboard on CIFAR-10 and a $20\times 20$ all-white patch on VGGFace2; The trigger is a Hello-Kitty image on CIFAR-10 and a random noise image on VGGFace2 for Blended. The transparency parameter is set as $0.2$; We exploit a pre-trained model to generate the targeted universal adversarial perturbation on the benign model as the trigger for TUAP. The $L_\infty$ constraint is set as $8/255$ on both datasets; For WaNet, the size of the control field is $4\times 4$ and the strength parameter of the backward warping field is set to $0.5$; We adopt the default settings used in its original paper \cite{li2021invisible} for ISSBA; UBW-P is an untargeted backdoor with the BadNets-type trigger pattern; For our SIBA, we set the step size $\alpha=0.2$ and $L_\infty$ constraint $\epsilon=8/255$ on both datasets. $L_0$ constraint $k$ is set to $100$ on CIFAR-10 dataset and $1,600$ on VGGface2 dataset. We set the number of training iterations $T=200$ and the update step $K=5$ on both datasets. We use the whole training set to optimize the SIBA trigger. We implement them based on \texttt{BackdoorBox} \cite{li2023backdoorbox}. The example of poisoned samples is shown in Figure \ref{fig:vary_cifar_backdoor}.  

\vspace{0.3em}
\noindent{\bf Training Settings.} We select ResNet-18 \cite{he2016deep} and VGG-16 \cite{simonyan2014very} as the model structures. For the CIFAR-10 dataset, the victim model is obtained by using SGD optimizer with momentum $0.9$ and weight decay $5\times 10^{-4}$. The number of training epochs is 100 and the initial learning rate is 0.1 which is divided by 10 in the 60-th epoch and the 90-th epoch. For the VGGFace2 dataset, we exploit an ImageNet pre-trained model and train the victim model using the SGD optimizer with momentum $0.9$ and weight decay $1\times 10^{-4}$. The number of training epochs is 30 and the initial learning rate is $0.001$ which is divided by 10 in the 15-th epoch and the 20-th epoch. Classical data augmentations such as random crop and random horizontal flip are used for higher benign accuracy. Note that the pre-trained model $f_{\bm b}$ and the victim model have the same network structure and we will explore the attack effectiveness when they are different in Section \ref{sec:transfer}.

\vspace{0.3em}
\noindent{\bf Evaluation Metrics.} Following the classical settings in existing research, we adopt benign accuracy (BA) and attack success rate (ASR) to evaluate all backdoor attacks. BA is the ratio of correctly classified samples in the benign test set. ASR is the ratio of test samples that are misclassified as the target class when the trigger is attached to them. BA indicates the model utility and ASR reflects the attack effectiveness. We report the comparison of the $L_0$ and $L_\infty$ distances of trigger patterns for indicating stealthiness. \textcolor{black}{Besides, we also exploit LPIPS \cite{zhang2018unreasonable} and the structural similarity index measure (SSIM) for the reference of stealthiness. The higher the BA, ASR, and SSIM, the lower the $L_0$, $L_\infty$, and LPIPS, the better the attack.} 

\begin{table*}[!t]
\centering
\caption{The results of backdoor attacks on CIFAR-10. \textcolor{black}{We mark the background of bad cases in red whose ASR is lower than $90\%$ or $L_0/ L_\infty$ distance is larger than $10\%$ of the maximum possible values. `--' denotes not available.}}
\vspace{-0.3em}
\scalebox{1.0}{
\begin{tabular}{c|c|c|cccccc|ccc}
\toprule
   Model$\downarrow$  &Metric$\downarrow$, Method$\rightarrow$ & No Attack & BadNets & Blended & TUAP & WaNet & ISSBA & UBW-P &  Random & Sparse &SIBA \\
\midrule
 \multirow{6}*{ResNet}& BA (\%) &94.67 &94.41 &94.58 &94.53 &94.29 &94.57 &94.46 &94.13 &94.44 &94.06 \\
 & ASR (\%) &-- & 100 & 98.16 & 85.47 & \cellcolor{magenta!20}47.90 &\cellcolor{magenta!20}0.76 &\cellcolor{magenta!20}64.78 &\cellcolor{magenta!20}2.87 &\cellcolor{magenta!20}88.49 & 97.60 \\
\cmidrule(r){2-2} \cmidrule(r){3-3} \cmidrule(r){4-9} \cmidrule(r){10-12} 
 & $L_0$ &-- & 9 &\cellcolor{magenta!20}1,020 &\cellcolor{magenta!20}1,020 &\cellcolor{magenta!20}1,016 &\cellcolor{magenta!20}1,024 & 9 & 100 & 100 & 100 \\
 & $L_\infty$ &-- &\cellcolor{magenta!20}0.81 &\cellcolor{magenta!20}0.19 & 0.03 &\cellcolor{magenta!20}0.19 & 0.05 &\cellcolor{magenta!20}0.81 & 0.03 & 0.03 & 0.03 \\
 \cmidrule(r){2-2} \cmidrule(r){3-3} \cmidrule(r){4-9} \cmidrule(r){10-12} 
 & LPIPS &-- &$<$ 0.001 &0.028 &0.001 &0.003 & $<$ 0.001 & $<$ 0.001 & $<$ 0.001 & $<$ 0.001 & $<$ 0.001 \\
 & SSIM &-- & 0.974 & 0.774 & 0.919 & 0.973 & 0.989  & 0.974 & 0.995  &  0.992 &  0.993\\
\midrule
 \multirow{6}*{VGG}& BA (\%) &93.34 &93.25 &93.31 &93.06 &92.09 &93.15 &93.16 &92.69 &92.59 &92.84 \\
 & ASR (\%) &-- & 100 & 98.28 &\cellcolor{magenta!20}69.13 &\cellcolor{magenta!20}7.53 &\cellcolor{magenta!20}1.24 &\cellcolor{magenta!20}10.02 &\cellcolor{magenta!20}1.48 &\cellcolor{magenta!20}66.44 & 91.22 \\
 \cmidrule(r){2-2} \cmidrule(r){3-3} \cmidrule(r){4-9} \cmidrule(r){10-12} 
 & $L_0$ &-- & 9 &\cellcolor{magenta!20}1,020 &\cellcolor{magenta!20}1,020 &\cellcolor{magenta!20}1,016 &\cellcolor{magenta!20}1,024 & 9 & 100 & 100 & 100 \\
 & $L_\infty$ &-- &\cellcolor{magenta!20}0.81 &\cellcolor{magenta!20}0.19 & 0.03 &\cellcolor{magenta!20}0.19 & 0.05 &\cellcolor{magenta!20}0.81 & 0.03 & 0.03 & 0.03 \\
 \cmidrule(r){2-2} \cmidrule(r){3-3} \cmidrule(r){4-9} \cmidrule(r){10-12} 
 & LPIPS &-- &$<$ 0.001 &0.028 &0.001 &0.003 & $<$ 0.001 & $<$ 0.001 & $<$ 0.001 & $<$ 0.001 & $<$ 0.001 \\
 & SSIM &-- & 0.974 & 0.774 & 0.919 & 0.973 & 0.989  & 0.974 & 0.995  &  0.992 &  0.993\\
 \bottomrule
\end{tabular}
}
\label{tab:main}
\end{table*}

\begin{table*}[!t]
\centering
\caption{The results of backdoor attacks on VGGFace2. \textcolor{black}{We mark the background of bad cases in red whose ASR is lower than $90\%$ or $L_0/ L_\infty$ distance is larger than $10\%$ of the maximum possible values. `--' denotes not available.}}
\vspace{-0.3em}
\scalebox{1.0}{
\begin{tabular}{c|c|c|cccccc|ccc}
\toprule
   Model $\downarrow$  &Metric $\downarrow$, Method $\rightarrow$ & No Attack & BadNets & Blended & TUAP & WaNet & ISSBA & UBW-P &  Random & Sparse &SIBA \\
\midrule
 \multirow{6}*{ResNet}& BA (\%) &79.75 &78.80 &78.80 &78.75 &79.05 &78.30 &77.85 &77.80 &78.25 &78.85 \\
 & ASR (\%) &-- & 93.68& 99.95 &\cellcolor{magenta!20}85.84 & \cellcolor{magenta!20}3.32 & \cellcolor{magenta!20}1.00 & \cellcolor{magenta!20}38.20 & \cellcolor{magenta!20}1.74 & \cellcolor{magenta!20}27.79 & 96.21 \\
 \cmidrule(r){2-2} \cmidrule(r){3-3} \cmidrule(r){4-9} \cmidrule(r){10-12} 
 & $L_0$ &-- & 400 & \cellcolor{magenta!20}16,384 & \cellcolor{magenta!20}16,332 & \cellcolor{magenta!20}15,892 & \cellcolor{magenta!20}16,382 & 400 & 1,600 & 1,600 & 1,600 \\
 & $L_\infty$ &-- & \cellcolor{magenta!20}0.92 & \cellcolor{magenta!20}0.20 & 0.03 & \cellcolor{magenta!20}0.23 & 0.04 & \cellcolor{magenta!20}0.92 & 0.03 & 0.03 & 0.03 \\
 \cmidrule(r){2-2} \cmidrule(r){3-3} \cmidrule(r){4-9} \cmidrule(r){10-12} 
 & LPIPS &-- &0.08 &0.18 &0.05 &0.02 & $<$ 0.01 &0.08 & $<$ 0.01 & $<$ 0.01 & $<$ 0.01 \\
  & SSIM &-- &  0.965  & 0.538 & 0.870 & 0.975 & 0.974  & 0.965 &  0.989 & 0.984  & 0.987 \\
\midrule
 \multirow{6}*{VGG}& BA (\%) &86.35 &86.15 &86.30 &85.90 &85.75 &85.40 &84.50 &85.90 &85.65 &86.15 \\
 & ASR (\%) &-- & 98.42 & 100 &\cellcolor{magenta!20}83.68 &\cellcolor{magenta!20}2.95 &\cellcolor{magenta!20}87.53 &\cellcolor{magenta!20}48.10 &\cellcolor{magenta!20}4.63 &\cellcolor{magenta!20}41.68 & 96.37 \\
 \cmidrule(r){2-2} \cmidrule(r){3-3} \cmidrule(r){4-9} \cmidrule(r){10-12} 
 & $L_0$ &-- & 400 &\cellcolor{magenta!20}16,384 &\cellcolor{magenta!20}16,332 &\cellcolor{magenta!20}15,892 &\cellcolor{magenta!20}16,382 & 400 & 1,600 & 1,600 & 1,600 \\
 & $L_\infty$ &-- &\cellcolor{magenta!20}0.92 &\cellcolor{magenta!20}0.20 & 0.03 &\cellcolor{magenta!20}0.23 & 0.04 &\cellcolor{magenta!20}0.92 & 0.03 & 0.03 & 0.03 \\
 \cmidrule(r){2-2} \cmidrule(r){3-3} \cmidrule(r){4-9} \cmidrule(r){10-12} 
  & LPIPS &-- &0.08 &0.18 &0.05 &0.02 & $<$ 0.01 &0.08 & $<$ 0.01 & $<$ 0.01 & $<$ 0.01 \\
  & SSIM &-- &  0.965  & 0.538 & 0.870 & 0.975 & 0.974  & 0.965 &  0.989 & 0.984  & 0.987  \\
 \bottomrule
\end{tabular}
}
\label{tab:main_vggface}
\end{table*}

\subsection{Main Results}
\label{sec:main_results}


As shown in Table \ref{tab:main}-\ref{tab:main_vggface}, our SIBA reaches the best performance among all sparse and invisible backdoor attacks ($i.e.$, Random, Sparse, and SIBA) on both datasets. Especially on the VGGFace2 dataset, the ASR improvements are larger than 50\% compared to Sparse and 90\% compared to Random. These results verify the effectiveness of our trigger optimization. Besides, the ASRs of our attack are always higher than 90\% and the BA decreases compared to the model trained without attacks are always less than 2\%. In particular, the attack performance of our SIBA is on par with (BadNets and Blended) or even better than (TUAP, WaNet, ISSBA, and UBW-P) of all baseline attacks that are either visible or not sparse. \textcolor{black}{Notably, SIBA is the only attack that achieves stealthiness across all four metrics ($L_0$, $L_\infty$, LPIPS, and SSIM).} These results verify the effectiveness and stealthiness of our proposed method. 

\textcolor{black}{Besides, we devised a survey requiring 20 participants to choose the most likely modified image within each set containing four clean images and one SIBA-poisoned image, where each survey has five image sets. The overall selection accuracy is $23\%$, which is close to that of a random guess (20\%). This result indicates that human eyes can hardly detect the SIBA trigger, verifying the stealthiness of our attack again.}


\begin{figure*}[!t]
    \centering
    \subfloat[CIFAR-10]{
\includegraphics[width=0.47\linewidth]{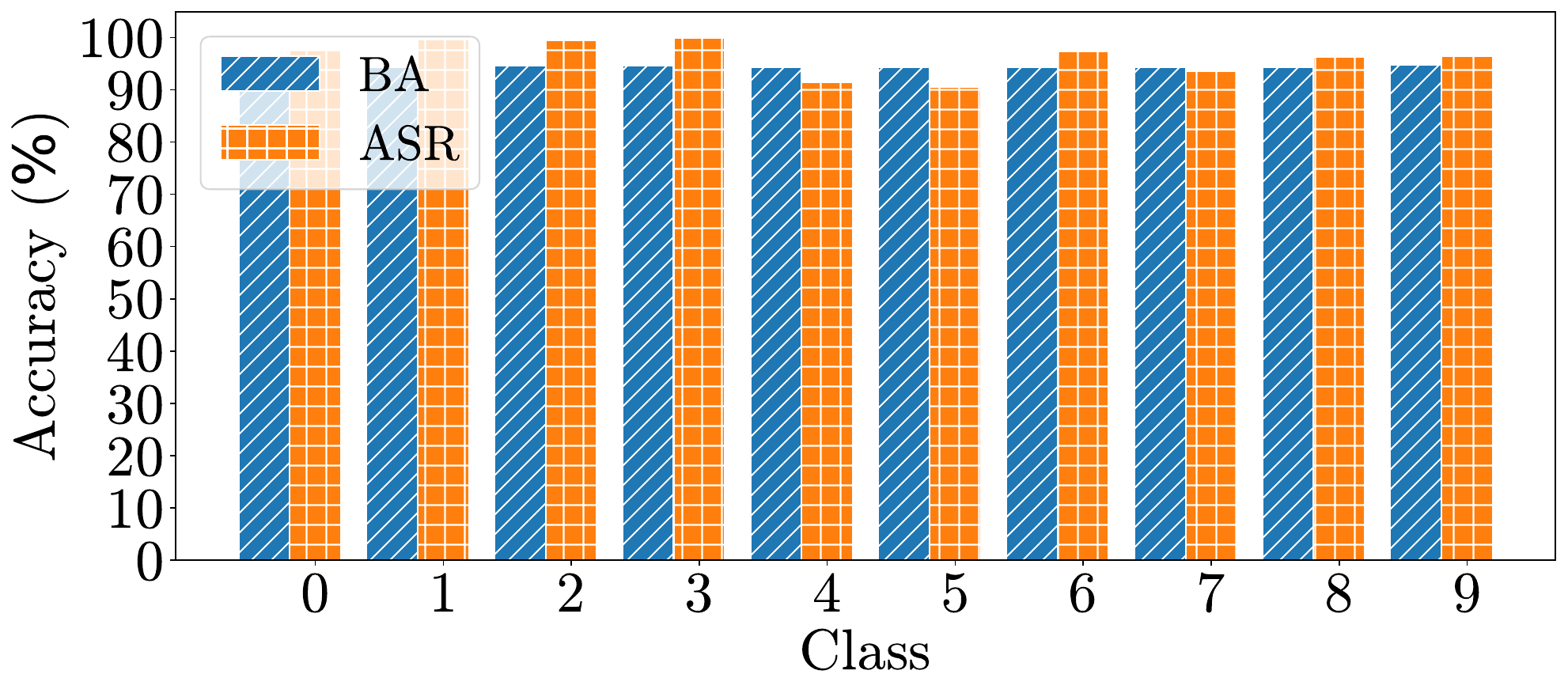}}
\hspace{0.5em}
\subfloat[VGGFace2]{
\includegraphics[width=0.47\linewidth]{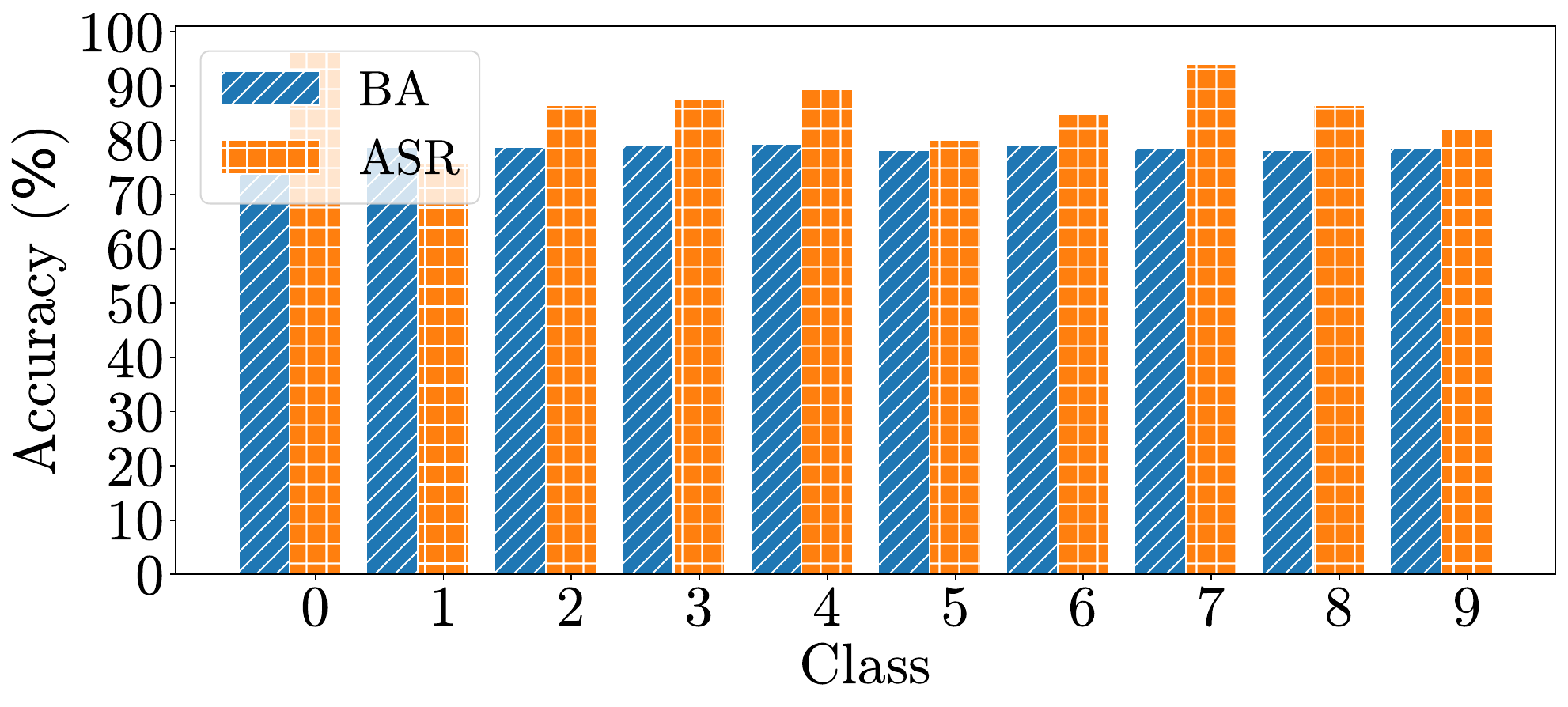}}
    \caption{The Effects of the target class on CIFAR-10 and VGGFace2 datasets.}
    \label{fig:vary_class}
\end{figure*}

\begin{figure}[!t]
    \centering
    \subfloat[CIFAR-10]{
     \includegraphics[width=0.48\linewidth]{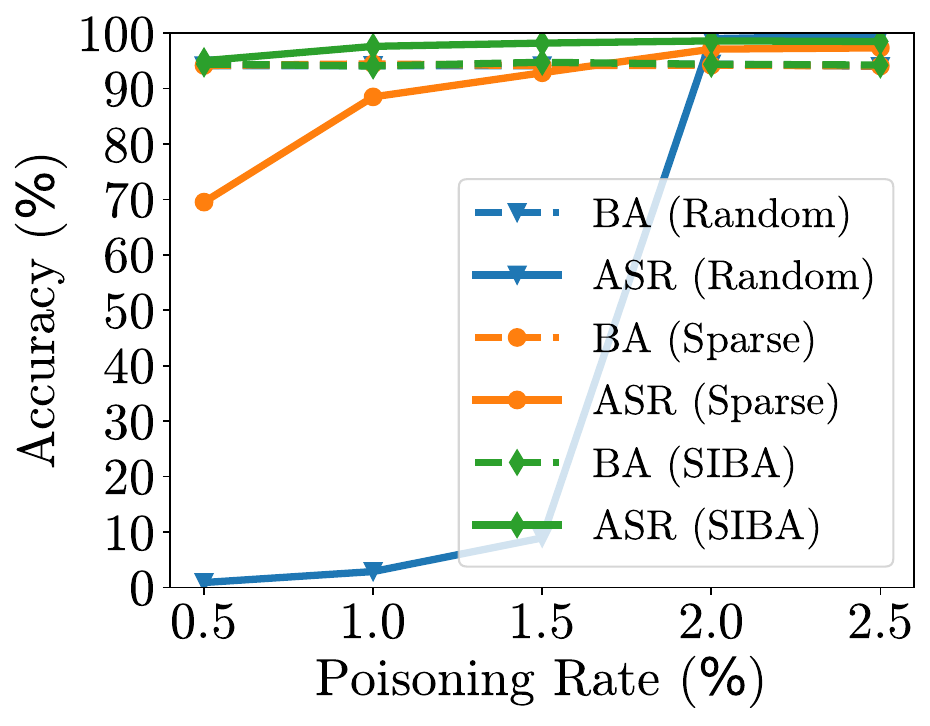}}
    \subfloat[VGGFace2]{
     \includegraphics[width=0.48\linewidth]{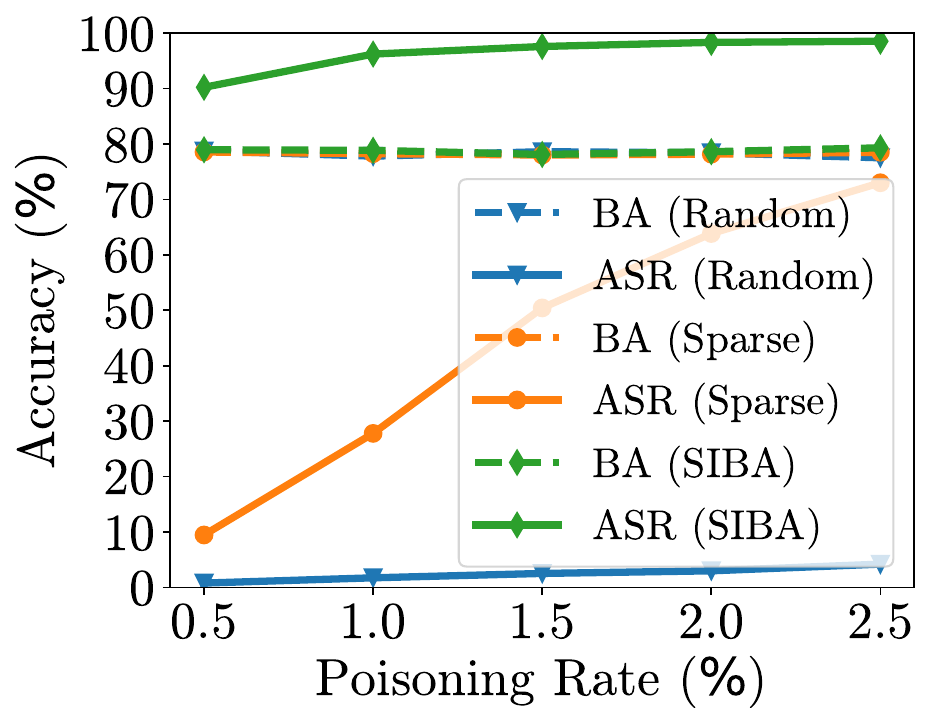}}
    \caption{Results with different poisoning rates on the CIFAR-10 dataset and the VGGFace2 dataset.}
    \label{fig:vary_pr}
    \vspace{-1em}
\end{figure}

\subsection{Ablation Study}

In this section, we discuss the effectiveness of our SIBA with different key hyper-parameters. Unless otherwise specified, all settings are the same as those used in Section \ref{sec:main_results}.

\vspace{0.3em}
\noindent{\bf Effects of the Target Label.} To validate the effectiveness of SIBA with different target labels, we conduct experiment on the ResNet18 with ten different classes. As shown in Figure \ref{fig:vary_class}, we could find that SIBA achieves $>90\%$ ASR for all cases on the CIFAR-10 dataset and $>75\%$ ASR on the VGGface2 dataset, although the performance may have some mild fluctuations.

\vspace{0.3em}
\noindent{\bf Effects of the Poisoning Rate.} To validate the effectiveness of SIBA with different poisoning rates, we experiment on the ResNet18 model with more poisoning rates from $0.5\%$ to $2.5\%$. As shown in Figure \ref{fig:vary_pr}, the attack performance of our SIBA increases with the increase of the poisoning rate. The attack performance of our SIBA is always better than baseline invisible and sparse attacks ($i.e.$, Random and Sparse). In particular, on the CIFAR-10 dataset, SIBA achieves $>90\%$ ASR with only $0.5\%$ poisoning rate while the poisoning rate of the other two baselines has to be set three or four times higher to achieve similar attack performance. The advantage of our SIBA is even more obvious on the VGGFace2 dataset.

\begin{figure}[!t]
    \centering
    \subfloat[CIFAR-10]{
     \includegraphics[width=0.48\linewidth]{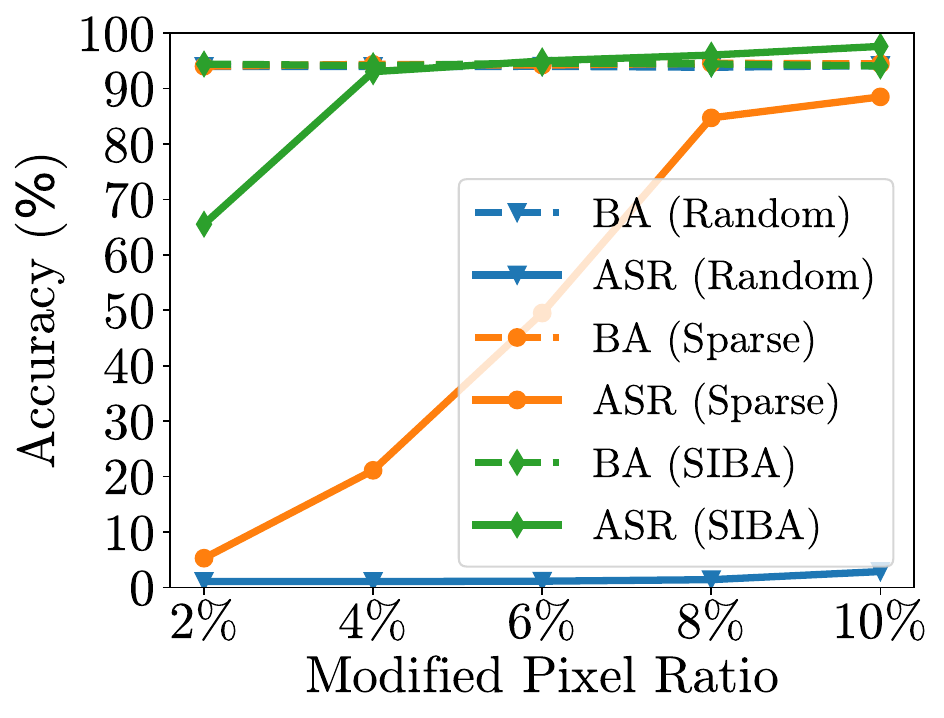}}
    \subfloat[VGGFace2]{
     \includegraphics[width=0.48\linewidth]{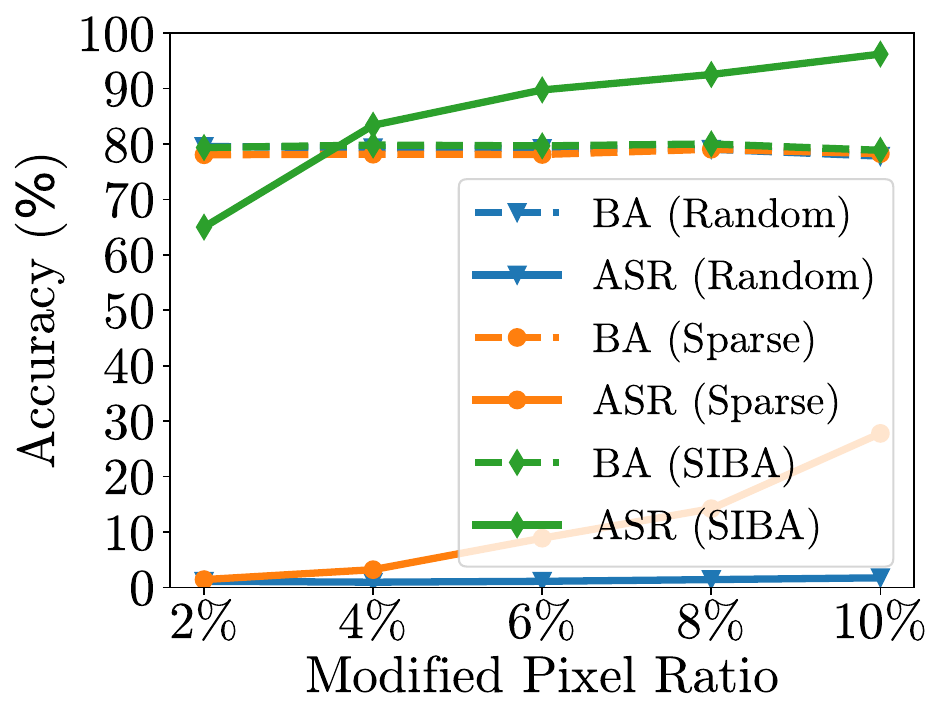}}
    \caption{\textcolor{black}{The effects of $L_0$ constraint on the CIFAR-10 dataset and the VGGFace2 dataset.}}
    \label{fig:vary_k}
    \vspace{-1em}
\end{figure}

\vspace{0.3em}
\noindent{\bf Effects of the $L_0$ Constraints.} To investigate the attack performance of our SIBA under various $L_0$ constraints, we experiment on the ResNet18 model with different $k$ values ranging from $50$ to $250$ on CIFAR-10 and from $800$ to $2400$ on VGGFace2, respectively. As shown in Figure \ref{fig:vary_k}, the attack effectiveness increases with the increase of $k$ while having mild effects on the benign accuracy. In particular, our SIBA achieves $>90\%$ ASR with only 50 perturbed pixels (about $5\%$ sparsity) on the CIFAR-10 dataset. However, for the other baseline attacks, the number of maximum perturbed pixels has to be increased to 150 for `Sparse' and 200 for `Random' to reach a similar performance. The improvement of our SIBA is even larger on the VGGFace2 dataset. 

\vspace{0.3em}
\noindent{\bf Effects of the $L_\infty$ Constraint.} To investigate the attack performance of our SIBA under various $L_\infty$ constraints, we experiment on the ResNet18 model with different $\epsilon$ values, ranging from $4/255$ to $20/255$. As shown in Figure \ref{fig:vary_epsilon}, similar to the effects of $k$, the attack effectiveness increases with the increase of $\epsilon$ while having mild effects on the benign accuracy. Our SIBA can achieve $>90$ ASR under $4/255$ budget on the CIFAR-10 dataset. In contrast, to achieve a similar attack performance, the $L_\infty$ constraints of both baseline attacks have to be increased to three or four times larger than that of the SIBA. The results on VGGFace2 also demonstrate the superiority of SIBA over these baseline methods.

\vspace{0.3em}
\noindent{\bf Effects of Other Parameters.} To demonstrate the stability of our attack under other parameters, we experiment on CIFAR-10 dataset with ResNet18 model and various $K$, $\alpha$, and $T$ values. Specifically, $K$ ranges from $5$ to $20$; $\alpha$ ranges from $0.2$ to $1.0$; $T$ ranges from $200$ to $1,000$. As shown in Figure \ref{fig:hyp_sen}, the ASR of SIBA is always higher than $95\%$ in all cases. These results indicate that we can easily obtain a good performance without fine-tuning these parameters in practice.

\begin{figure}[!t]
    \centering
    \subfloat[CIFAR-10]{
     \includegraphics[width=0.48\linewidth]{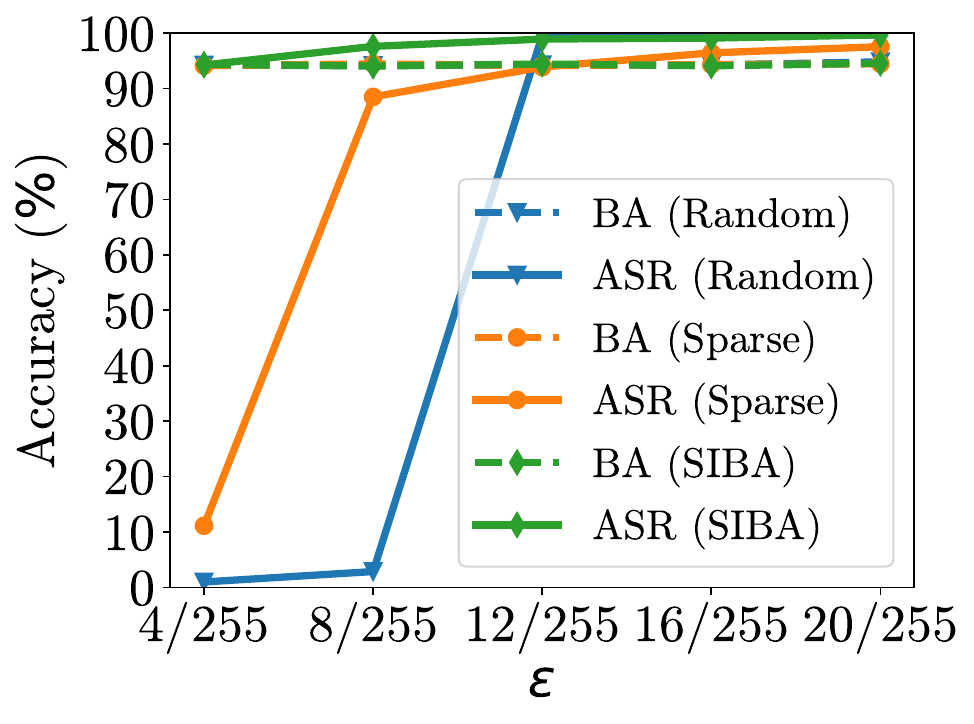}}
    \subfloat[VGGFace2]{
     \includegraphics[width=0.48\linewidth]{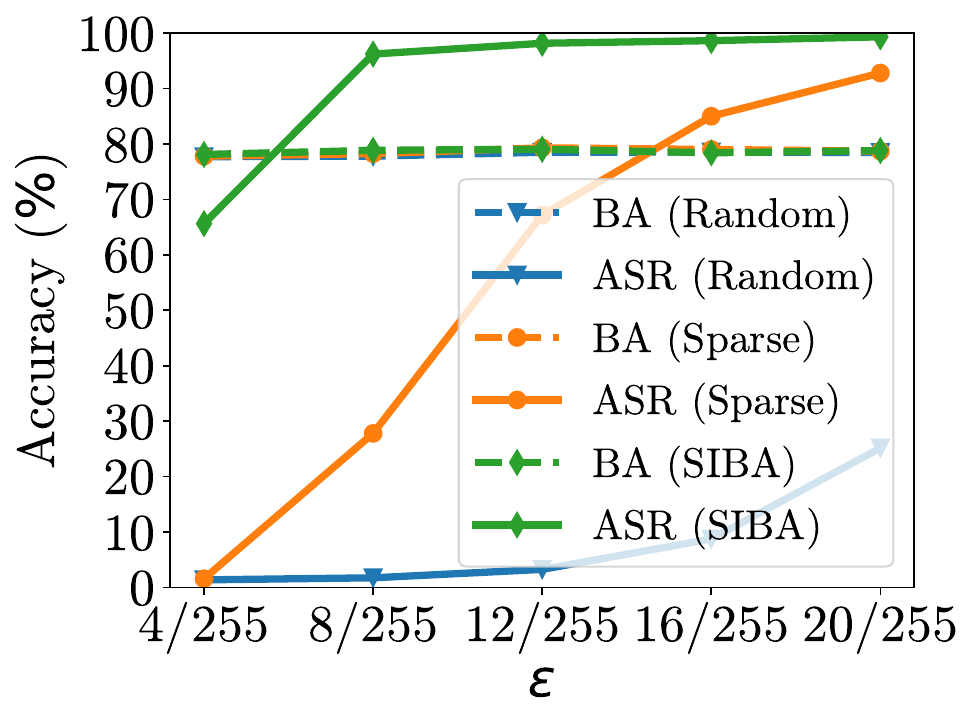}}
    \caption{The effects of $L_\infty$ constraint on the CIFAR-10 dataset and the VGGFace2 dataset.}
    \label{fig:vary_epsilon}
\end{figure}

\subsection{The Resistance to Potential Backdoor Defenses}


\vspace{0.3em}
\noindent{\bf The Resistance to STRIP.} As a representative black-box detection of poisoned training samples with predicted logits, STRIP \cite{gao2021design} perturbs a given test image by superimposing various images and then inspects the entropy of the model prediction. The suspicious samples having low entropy are regarded as poisoned samples. We evaluate the resistance of our SIBA to STRIP by visualizing the entropy distributions of samples. As shown in Figure \ref{fig:strip}, the entropy distributions of poisoned samples are mixed with those of benign samples. Accordingly, our SIBA can evade the detection of STRIP.

\vspace{0.3em}
\noindent{\bf The Resistance to Anti-backdoor Learning (ABL).} As a representative poison suppression method, ABL \cite{li2021anti} first identifies the poisoned sample candidates with loss values and then unlearns the candidate samples by gradient ascent. In our experiments, the isolation epoch and the unlearning epoch are set to 20 and 80, as suggested in its original paper \cite{li2021anti}. As shown in Table \ref{tab:abl}, our attack is resistant to ABL in most cases, although the ASR may have some decreases. Its failure is mostly because the loss values are not effective to reflect the difference between poisoned and benign samples of SIBA.  


\begin{table}[!t]
\centering
\caption{The resistance to anti-backdoor learning (ABL).}
\vspace{-0.5em}
\begin{tabular}{c|c|cc}
\toprule 
Dataset$\downarrow$ & Model$\downarrow$, Metric$\rightarrow$
& BA (\%) & ASR (\%)  \\
\hline 
\multirow{2}*{CIFAR-10}& ResNet & 88.33 & 22.71  \\
 &VGG & 82.41 & 94.00 \\
\cmidrule{1-4}
\multirow{2}*{VGGFace2}& ResNet & 72.45 & 74.74 \\
 &VGG & 77.20 & 96.53 \\
\bottomrule 
\end{tabular}
\label{tab:abl}
\end{table}

\begin{figure*}
    \centering
    \subfloat[Sensitivity of $K$.]{
     \includegraphics[width=0.3\linewidth]{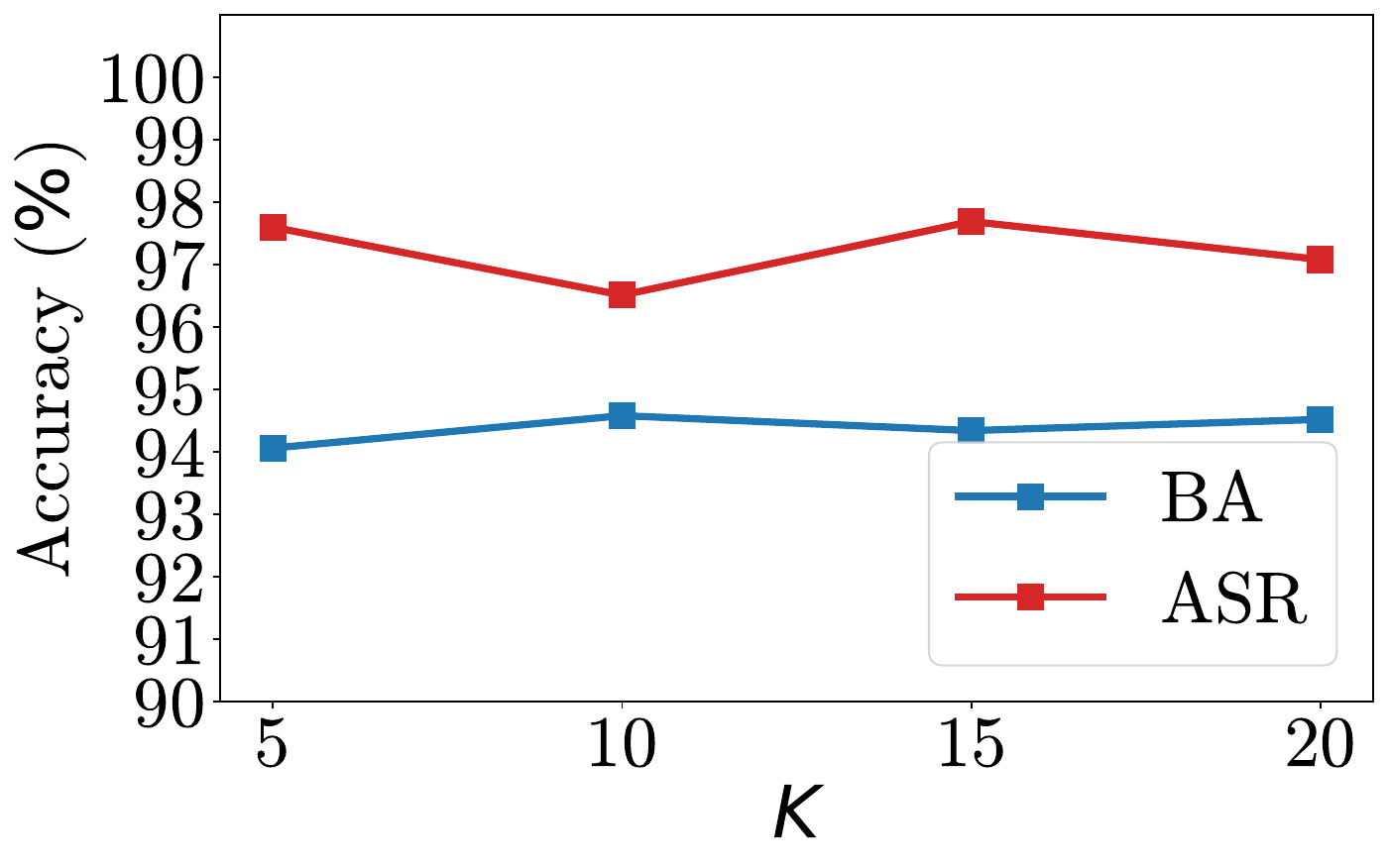}}
     \hspace{0.8em}\subfloat[Sensitivity of $\alpha$.]{\includegraphics[width=0.3\linewidth]{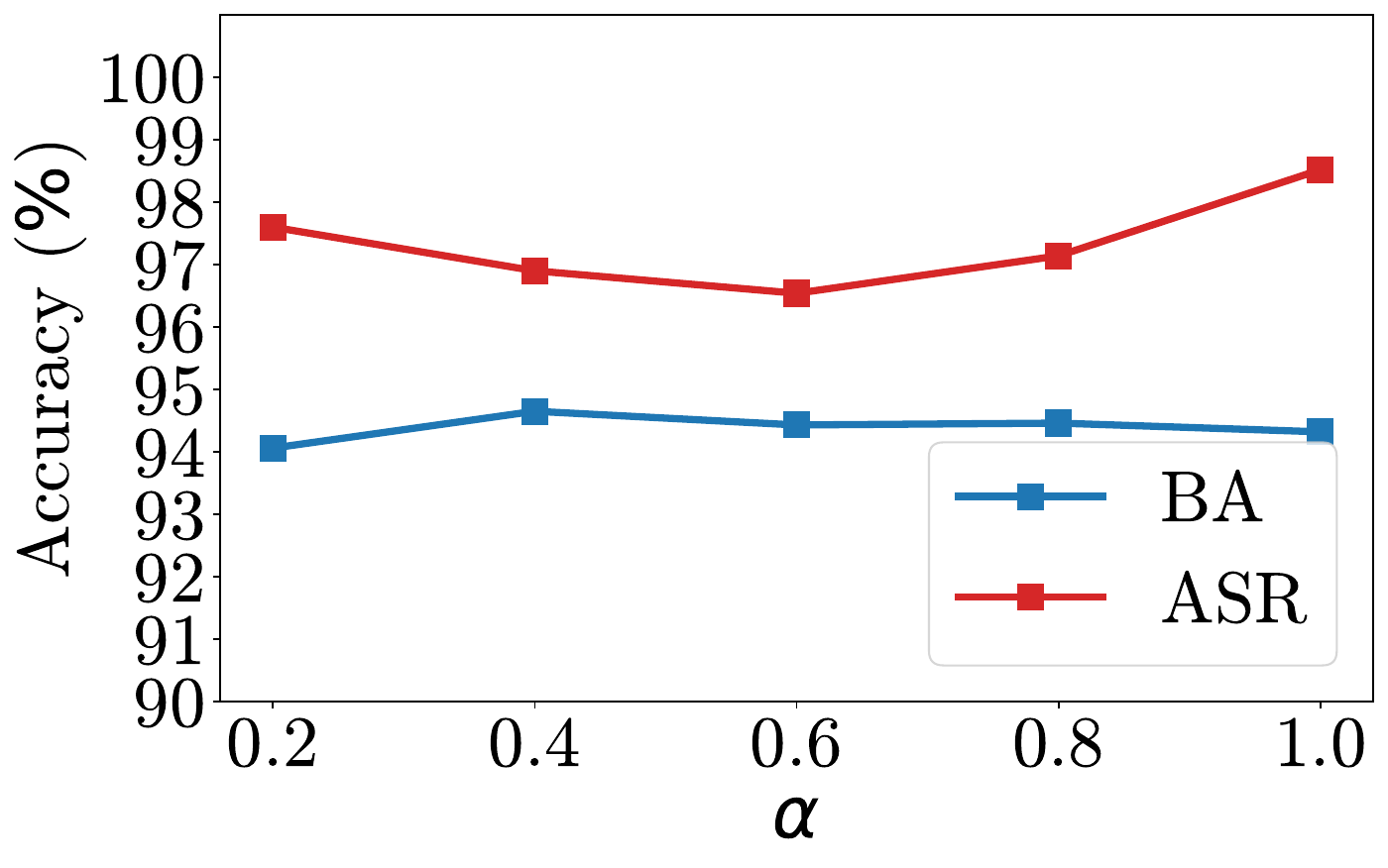}}
    \hspace{0.8em}\subfloat[Sensitivity of $T$.]{
     \includegraphics[width=0.3\linewidth]{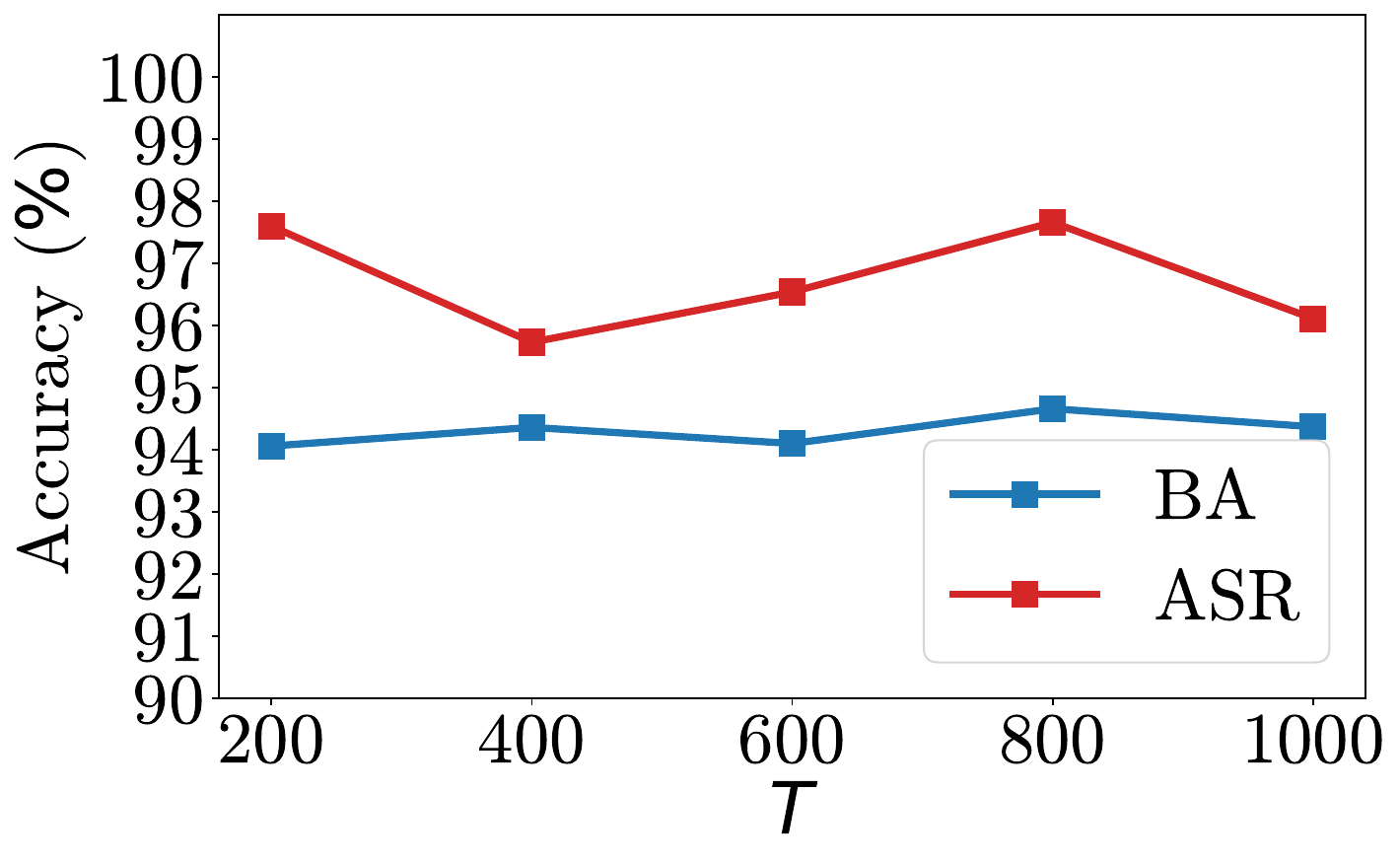}}
    \caption{Results of our SIBA with different parameters on the CIFAR-10 dataset and the VGGFace2 dataset.}
    \label{fig:hyp_sen}
    \vspace{-0.5em}
\end{figure*}

\vspace{0.3em}
\noindent {\bf The Resistance to Fine-pruning (FP).} As a representative backdoor removal method, fine-pruning (FP) \cite{liu2018fine} first tests the candidate model with a small clean validation set and records the average activation of each neuron. Then, FP prunes the channels with increasing order until the clean accuracy drops below some threshold. In our experiments, the validation set is obtained by randomly choosing $20\%$ samples from the clean training dataset and the total channel number is 512. The curves of BA and ASR with respect to the number of pruned channels are shown in Figure \ref{fig:fp}. We could observe that the ASR of the proposed backdoor attack preserves on CIFAR-10 even if a large portion of channels are pruned. As for the VGGFace2 dataset, the ASR is reduced below $80\%$ when the number of pruned channels is larger than $400$. However, the BA is significantly decreased as its sacrifice. These results verify the resistance of our SIBA to FP.


\begin{figure*}
    \centering
    \includegraphics[width=0.9\linewidth]{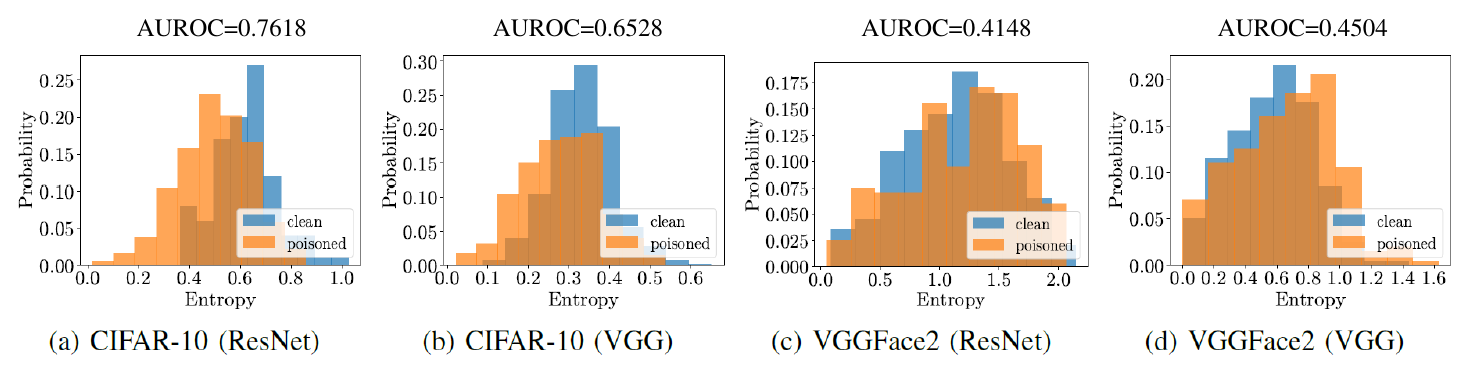}
    \caption{\textcolor{black}{The resistance of our SIBA to STRIP.} }
    \label{fig:strip}
    \vspace{-1.5em}
\end{figure*}

\begin{figure*}[!t]
    \centering
    \subfloat[CIFAR-10 (ResNet)]{
     \includegraphics[width=0.22\linewidth]{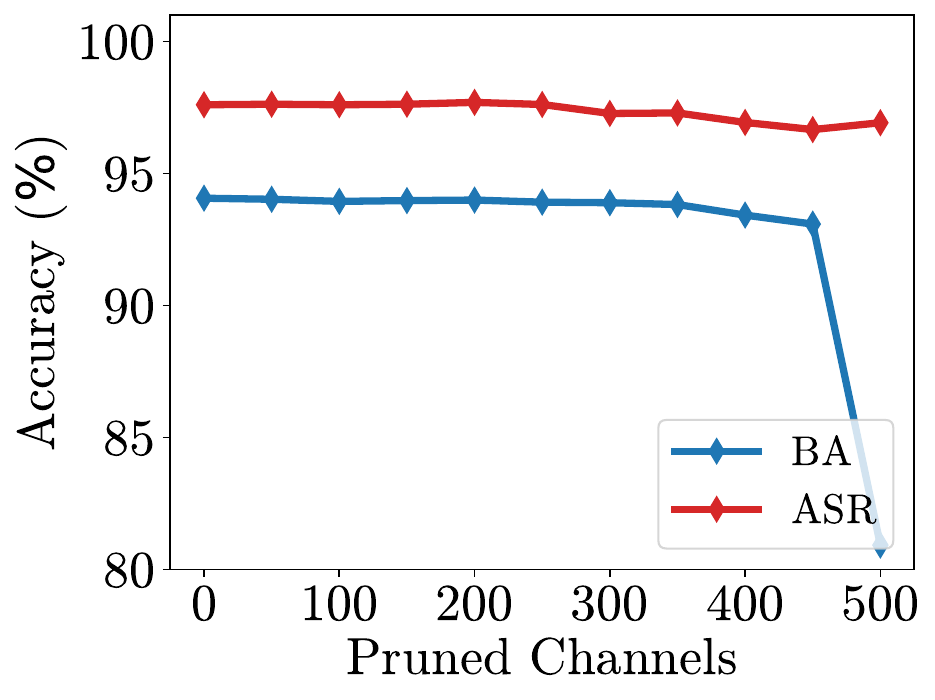}}
    \subfloat[CIFAR-10 (VGG)]{
     \includegraphics[width=0.22\linewidth]{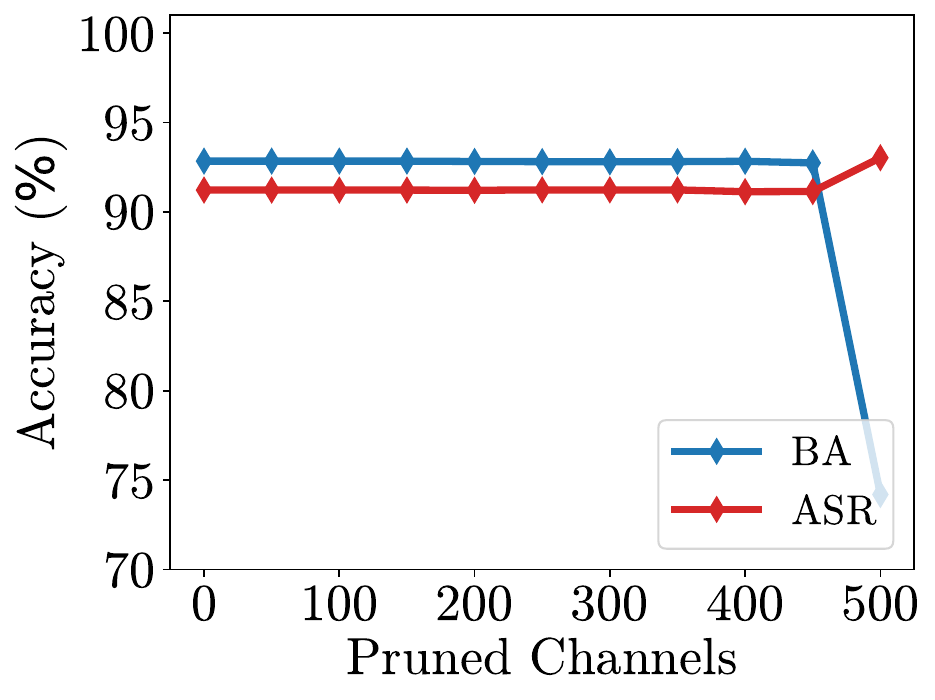}}
    \subfloat[VGGFace2 (ResNet)]{
     \includegraphics[width=0.22\linewidth]{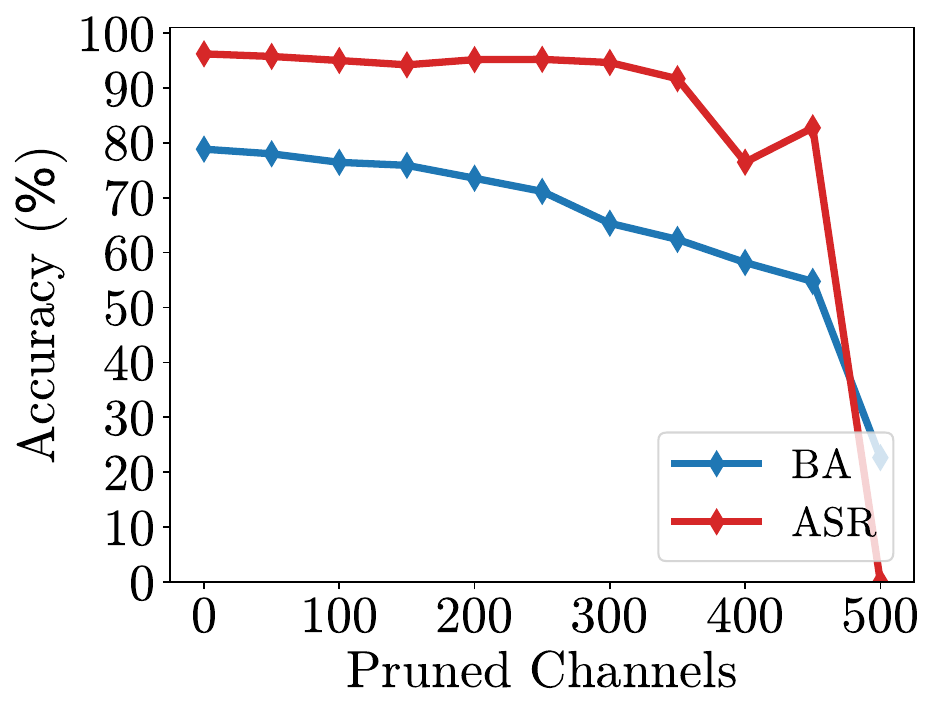}}
    \subfloat[VGGFace2 (VGG)]{
     \includegraphics[width=0.22\linewidth]{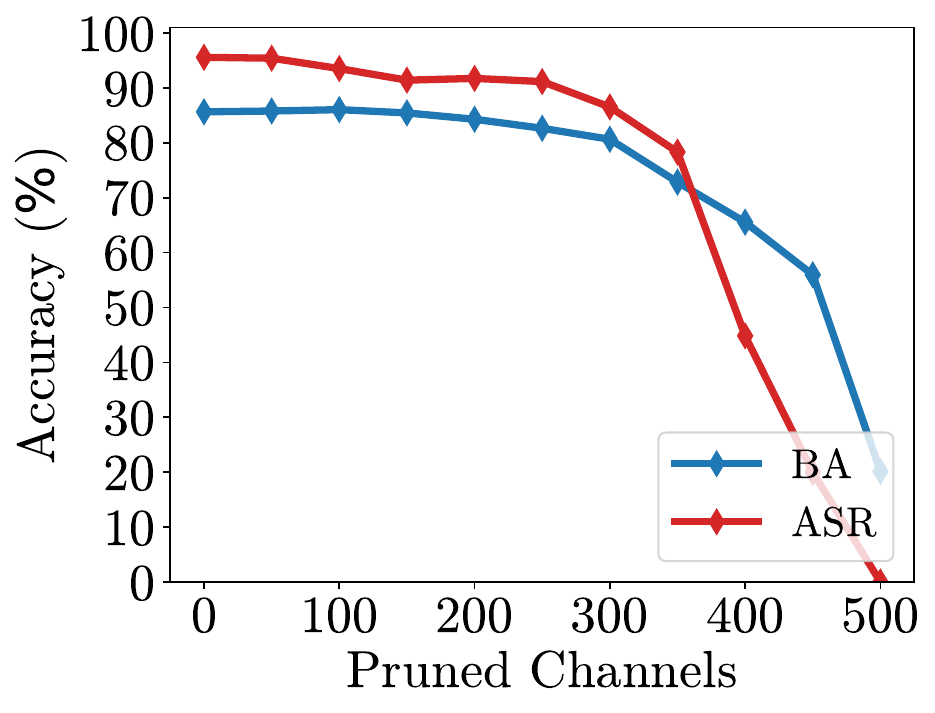}}
     \caption{The resistance of our SIBA to fine-pruning (FP).}
    \label{fig:fp}
    \vspace{-0.5em}
\end{figure*}



\begin{figure*}
    \centering
    \subfloat[CIFAR-10]{
     \includegraphics[width=0.9\linewidth]{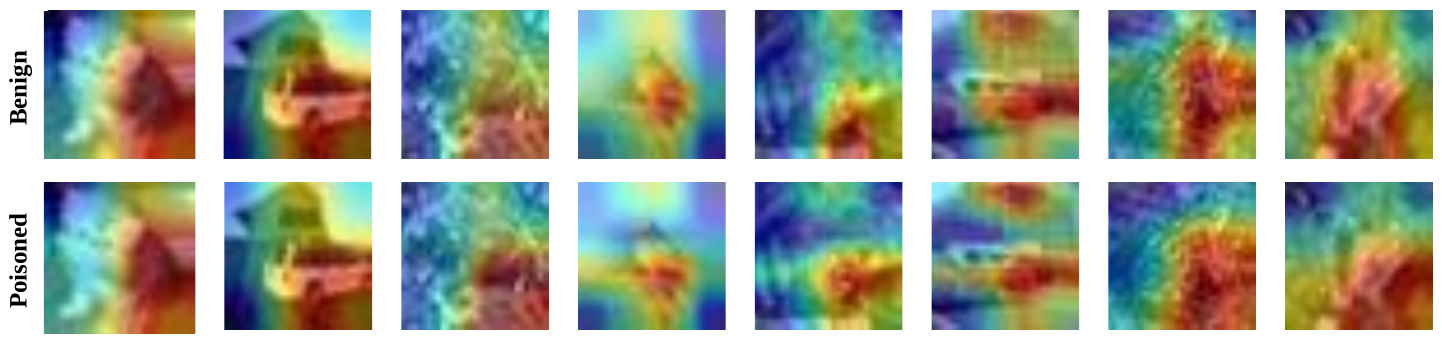}}
    
    \subfloat[VGGFace2]{
     \includegraphics[width=0.9\linewidth]{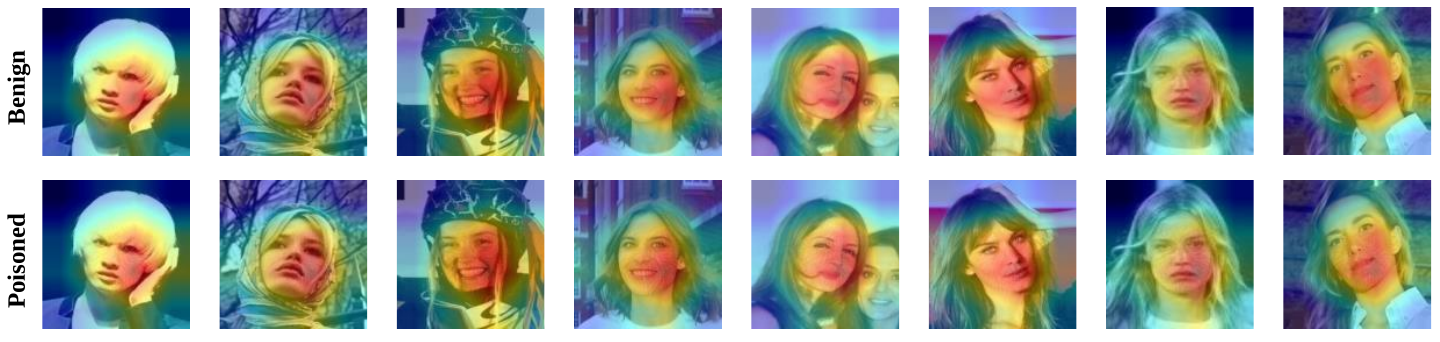}}
     \caption{The resistance of our SIBA to SentiNet. }
    \label{fig:cam}
\end{figure*}

\begin{figure}[!t]
    \centering
    \vspace{-1em}
\includegraphics[width=0.9\linewidth]{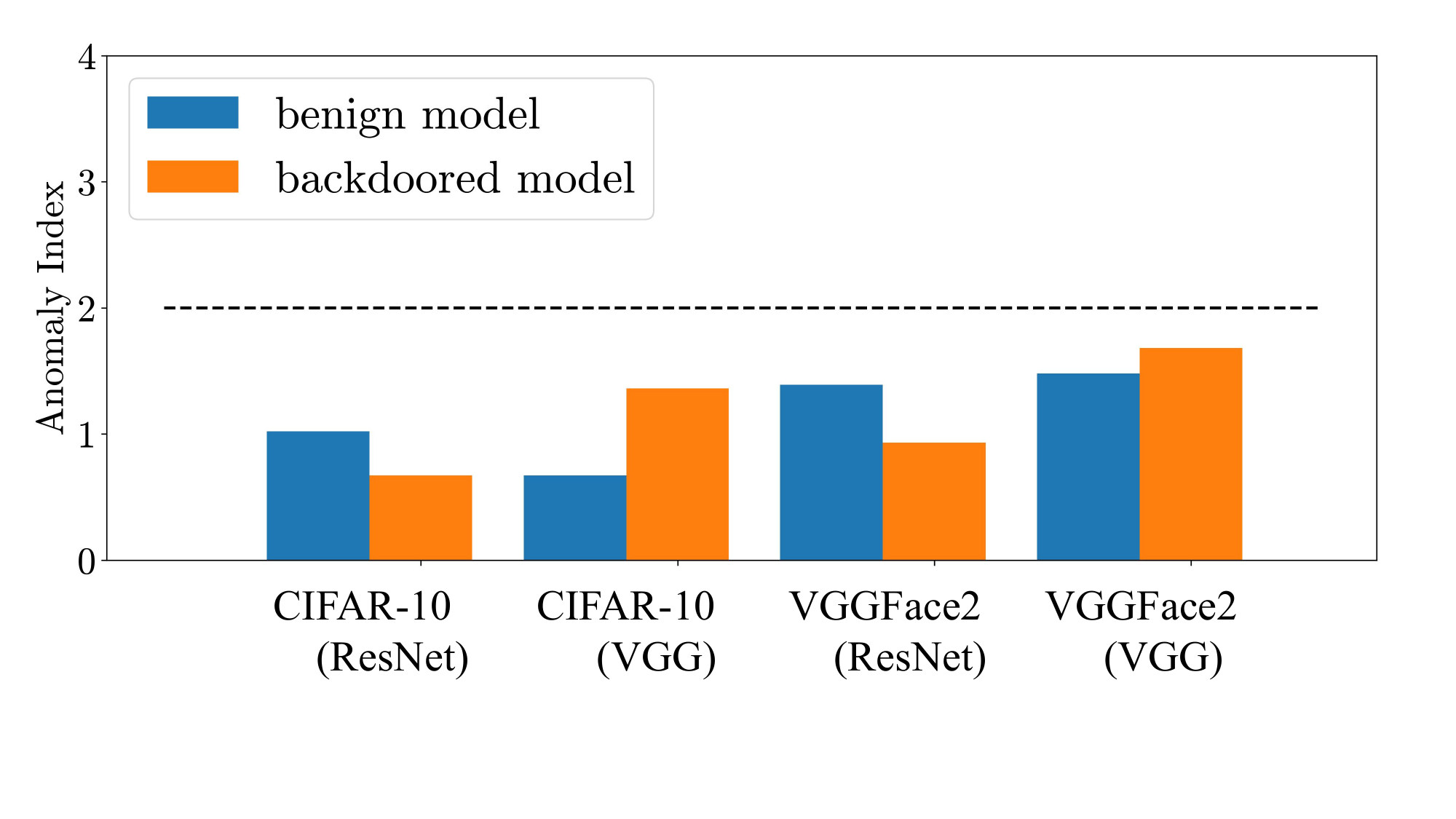}
     \vspace{-2.5em}
    \caption{The resistance of our SIBA to neural cleanse (NC).}
    \label{fig:nc}
    \vspace{-1em}
\end{figure}

\begin{table}[!t]
\centering
\caption{The resistance of our SIBA to Scale-Up. }
\vspace{-0.5em}
\scalebox{0.9}{
\begin{tabular}{c|c|cccc}
\toprule 
Dataset$\downarrow$ & Model$\downarrow$, Metric$\rightarrow$
& TPR & FPR & AUROC& ASR (\%) \\
\hline
\multirow{2}*{CIFAR-10} & ResNet & 0.5816 & 0.2067 & 0.7427 & 41.84  \\
&VGG & 0.4028 & 0.2135 &0.6369 & 59.72 \\
\midrule\multirow{2}*{VGGFace2} & ResNet& 0.0220  & 0.0855  &0.4428& 97.80 \\
& VGG & 0.0295 & 0.0830 & 0.4603 & 97.05 \\
\bottomrule 
\end{tabular}}
\label{tab:scale_up}
\vspace{-1em}
\end{table}

\vspace{0.3em}
\noindent{\bf The Resistance to Scale-Up.} As a representative black-box detection of poisoned testing samples with predicted labels, Scale-Up \cite{guo2023scale} discovered the phenomenon that the poisoned samples had the scaled prediction consistency when the pixel values were amplified and proposed to distinguish the poisoned samples by counting the predictions of scaled images. In our experiments, we use a scaling set with size 5 and set the threshold as $0.8$. We report true positive rate (TPR), false positive rate (FPR), area under the receiver operating characteristic (AUROC), and ASR in Table \ref{tab:scale_up}. As shown in the table, although Scale-Up can decrease the effectiveness of SIBA to some extent, the detection performance is far from satisfactory since the average ASR is still $>70\%$. In other words, our SIBA is resistant to the Scale-Up to a large extent.

\vspace{0.3em}
\noindent{\bf The Resistance to SentiNet.} As a representative white-box detection of poisoned testing samples, SentiNet \cite{chou2020sentinet} relies on model interpretability techniques to locate potential trigger regions. Grad-CAM uses the gradient with respect to the model's final layer and calculates the salience map of the input region to reflect the positive importance of the input image. In our experiments, we visualize the salience maps of some poisoned samples on CIFAR-10 and VGGFace2 datasets. As shown in Figure \ref{fig:cam}, the salience maps could not provide useful information to detect the trigger. Its failure is mostly because the trigger of SIBA is not a small-sized patch.

\vspace{0.3em}
\noindent {\bf The Resistance to Neural Cleanse (NC).} \textcolor{black}{As a representative model-level detection method, NC \cite{wang2019neural} first reverses possible triggers of the suspicious model and collects the $L_1$ values of the optimized trigger associated with each target label. Then, NC calculates the median absolute deviation of the group and the anomaly index of each label. If the anomaly index is larger than the threshold, the model is regarded as backdoored.} In our experiments, the threshold is 2 as suggested in its original paper. We use the Adam optimizer in which the learning rate is $0.1$. The coefficient of the regularization term is $0.001$ and the number of training epochs is 50. As shown in Figure \ref{fig:nc}, the NC is ineffective to detect our backdoored model since the anomaly index of the proposed method is always lower than the threshold value. This failure is mostly because our SIBA only needs to manipulate a small number of pixels such that the optimization of NC cannot catch our trigger location.

\vspace{0.3em}
\noindent \textcolor{black}{\bf The Resistance to Meta Neural Trojan Detection (MNTD).} \textcolor{black}{The MNTD \cite{xu2021detecting} is another famous model-level detection method. It trains a meta-classifier to determine whether a target model is backdoored. In our experiments, we set the query number 10 and use Adam optimizer \cite{kingma2014adam} with a 0.001 learning rate to train the meta classifier. Besides, we train 20 candidate models (10 SIBA-backdoored models and 10 clean models) and then calculate the trojan scores based on its meta classifier. The result (AUROC=$0.57$) indicates that SIBA can escape from MNTD. It is mostly because MNTD cannot capture sufficient trigger-related features due to our trigger sparseness. We will study it further in our future works.}



\subsection{Discussion}

\begin{figure}[!t]
    \centering
    \vspace{-1em}
    \subfloat[Airplane]
    {
    \includegraphics[width=0.33\linewidth]{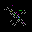}}
    \hspace{3em}
    \subfloat[Bird]{
     \includegraphics[width=0.33\linewidth]{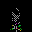}}

    \subfloat[Horse]{
     \includegraphics[width=0.33\linewidth]{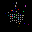}}
     \hspace{3em}
    \subfloat[Truck]{
     \includegraphics[width=0.33\linewidth]{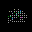}}
    \caption{The normalized SIBA triggers generated with different target classes on the CIFAR-10 dataset.}
    \label{fig:norm_tri}
\end{figure}

\subsubsection{\textcolor{black}{A Closer Look to the Effectiveness of our SIBA and its Connection to DNN Interpretability}}



\textcolor{black}{To understand the effectiveness of our SIBA, we also visualize the (normalized) trigger patterns when different target labels are adopted. As shown in Figure \ref{fig:norm_tri}, the SIBA trigger is always located in the main body of the object from the target class (such as the airplane's fuselage and wings in Figure \ref{fig:norm_tri}). In other words, the generated sparse trigger pattern is closely related to the concept of the target label. This phenomenon also (partially) explains the prediction and learning mechanisms of DNNs. Specifically, from the prediction perspective, it indicates that the misclassification of DNNs towards poisoned samples from their ground-truth label to the target one might be attributed to perturbations in regions associated with the target class. From the learning perspective, it implies that DNNs learn the trigger pattern as the representative of samples from the target class since we only re-assign the label of selected samples as the target label when generating the trigger pattern. It suggests that our attack may be a viable path toward understanding the learning and prediction principles of DNNs.} \textcolor{black}{Our results build an interesting connection between backdoor attacks and DNN interpretability that has not been shown before.}



\begin{figure}[!t]
    \centering
    \subfloat[BA]{
\includegraphics[width=0.46\linewidth]{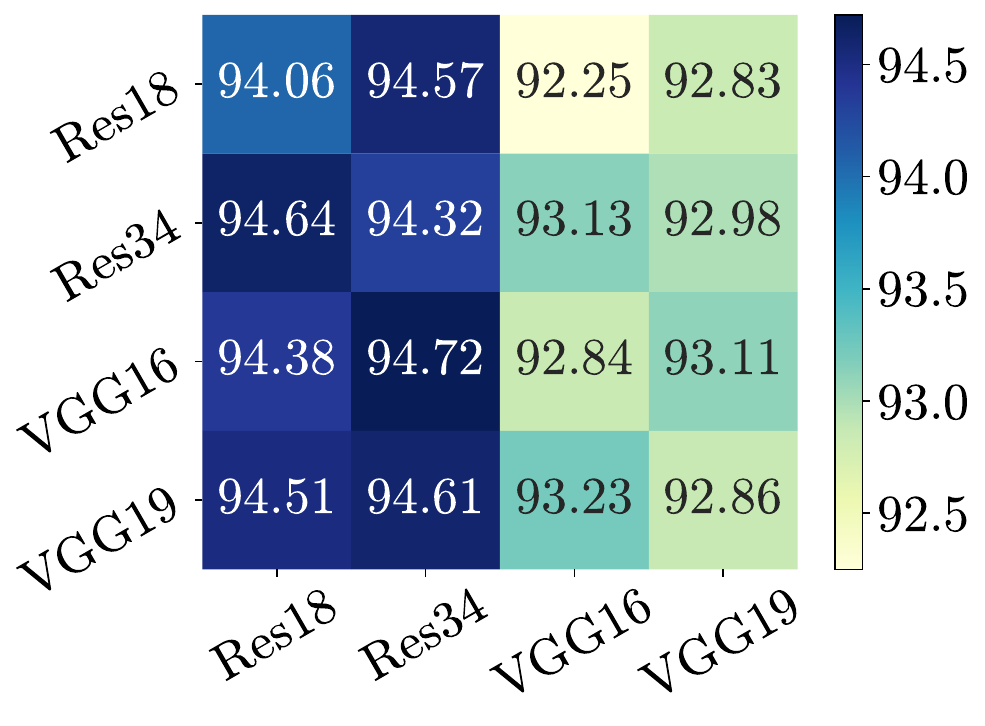}}
\hspace{1em}
    \subfloat[ASR]{
\includegraphics[width=0.46\linewidth]{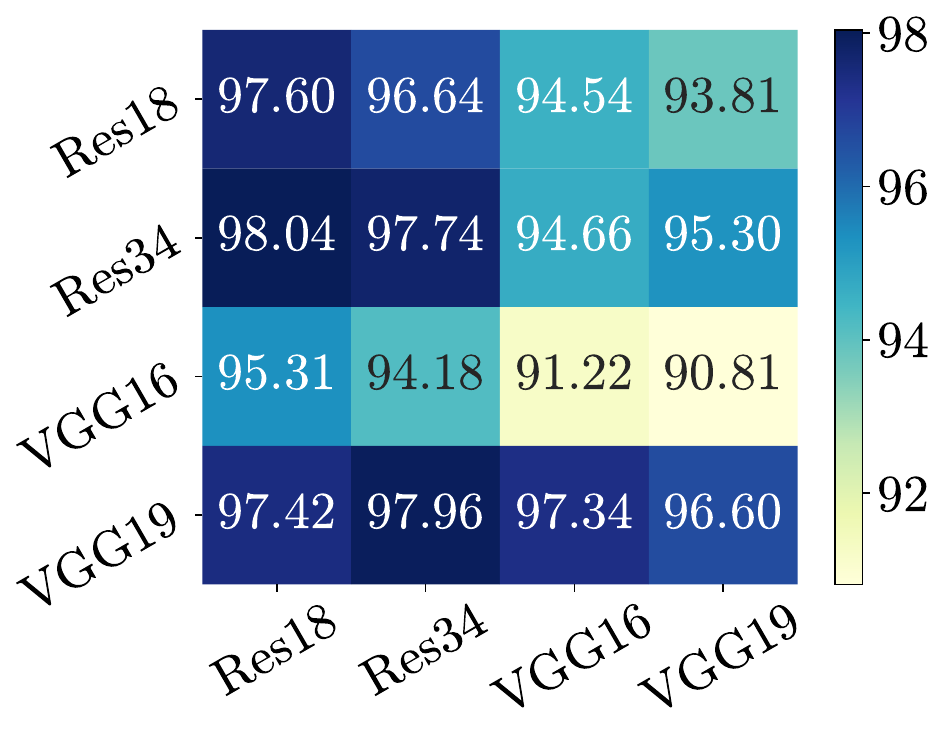}}
    \caption{The performance of our SIBA with different surrogate and victim structures on the CIFAR-10 dataset. \textbf{Row}: surrogate models; \textbf{Column}: victim models.}
    \label{fig:cifar_transfer}
\end{figure}

\vspace{0.3em}
\subsubsection{Attack Transferability with Different Model Structures}\label{sec:transfer} As described in Section \ref{sec:SIBA_method}, we need a pre-trained benign model to generate our SIBA trigger pattern. The experiments in Section \ref{sec:main_exp} are conducted based on the assumption that the surrogate model and victim model have the same model structure, which may not be feasible in practice since the adversaries have no information of the structure that victim users may use. In this part, we explore the transferability of our SIBA: `\textit{How effective is SIBA when the surrogate model is different from the victim model? }'. We select four network architectures: ResNet18, ResNet34, VGG16, and VGG19 on CIFAR-10 for discussions. Other settings are the same as those used in Section \ref{sec:main_results}. As shown in Figure \ref{fig:cifar_transfer}, our SIBA achieves consistently excellent attack performance under different settings, although the performance may have some fluctuations due to different model capacities. These results indicate that our SIBA method does not require knowing any information of victim users and therefore can serve as an effective poison-only backdoor attack.


\subsubsection{The Extension to All-to-all Setting} The experiments in Section \ref{sec:main_results} adopt the all-to-one setting, \textcolor{black}{where all poisoned samples are expected to be classified as the same target class.} In this part, we extend our SIBA to the all-to-all setting, in which the target class depends on the ground truth class of the poisoned sample. Specifically, we adopt the most classical transformation function `$c(y)=(y+1)\bmod C$' in this paper, \textcolor{black}{following the settings of existing papers \cite{gu2019badnets, doan2022marksman}.} In this case, the problem formulation of our SIBA is as follows: 
\begin{equation}
\begin{aligned}
    \min_{\bm t} \sum_{({\bm x},y)\in \mathcal{D}_v} &\mathcal{L}(f_{\bm b}({\bm x}+{\bm t}), c(y)) \\
    s.t. \ \lVert {\bm t} \rVert_0 \leq k&, \ \lVert {\bm t} \rVert_\infty \leq \epsilon.
\label{eq:a2a}
\end{aligned}
\end{equation}
We conduct experiments of the all-to-all SIBA attack on the CIFAR-10 dataset with ResNet18. The poisoning rate is increased to $10\%$ since the all-to-all attacks are more complicated than all-to-one methods. All other settings are the same as those used in Section \ref{sec:main_results}. As a result, the BA is $94.61\%$ and the ASR is $93.34\%$, indicating that our SIBA is feasible to be applied under the all-to-all setting.

\vspace{0.3em}
\subsubsection{SIBA with Limited Training Data}\label{sec:limit_data} In the previous sections, we assume that the adversary optimizes the SIBA trigger via the whole training set. However, in real scenarios, it might be infeasible to acquire the whole training set for the adversary to train the surrogate model. We hereby raise the question:`\textit{How effective is SIBA when the adversary has limited data?}'. In this part, we optimize our SIBA trigger based on a subset of the training set in which the data percentage ranges from $5\%$ to $20\%$. We report the BA of the surrogate model, and the BA and the ASR of the victim model. As shown in Table \ref{tab:diss_sub}, the degraded performance of the surrogate model does not mean the inefficiency of SIBA when only limited training data is adopted. Our SIBA achieves $>90\%$ ASR even when the adversary can only access to $10\%$ training data. These results verify the efficiency of our SIBA.

\begin{figure}
    \centering
\includegraphics[width=\linewidth]{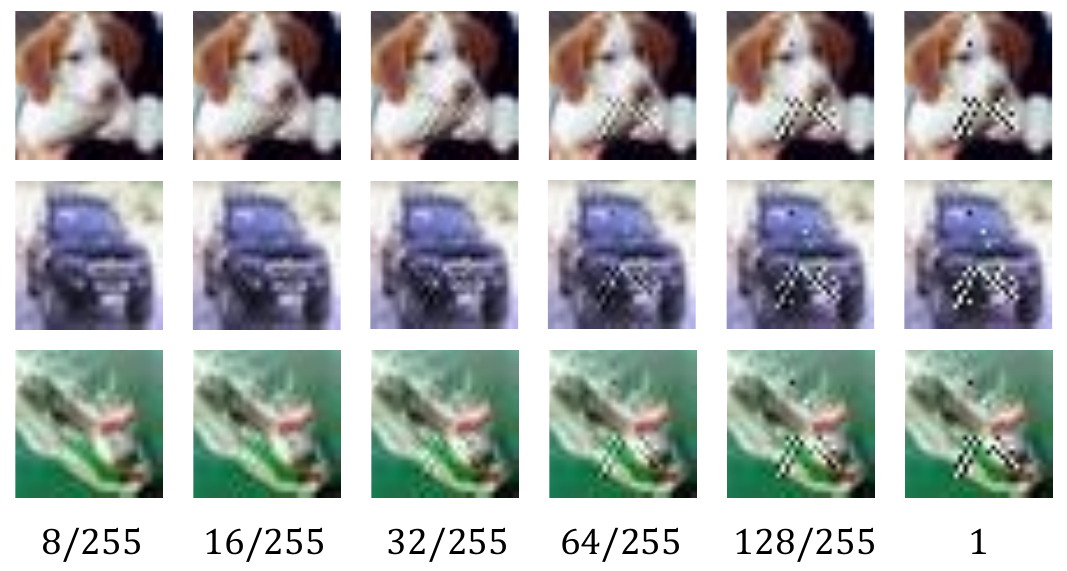}
    \caption{The example of poisoned samples with amplified triggers on the CIFAR-10 dataset.}
    \label{fig:amp_tri_sample}
\end{figure}

\begin{table}[!t]
\caption{Results with limited data on the CIFAR-10 dataset. }
\vspace{-0.3em}
\centering
\scalebox{1}{
\begin{tabular}{c|ccc}
\toprule 
\tabincell{c}{Metric$\rightarrow$\\
Data Percentage$\downarrow$}
& BA (Surrogate) & BA (Victim) & ASR (Victim) \\
\hline 
5\% & 57.55\% & 94.44\%  & 65.44\%  \\
10\% & 70.65\% & 94.25\% & 97.10\% \\
15\% & 79.32\%  & 94.11\%  & 95.49\% \\
20\% & 84.14\% & 94.83\% & 97.56\% \\
\hline
100\% & 94.67\% & 94.06\% & 97.60\% \\
\bottomrule 
\end{tabular}}
\label{tab:diss_sub}

\end{table}

\vspace{0.3em}
\subsubsection{SIBA with Asymmetric Triggers}
To further boost the attack effectiveness while maintaining attack stealthiness in practical scenarios, we explore the idea of asymmetric triggers \cite{chen2017targeted, qi2023revisiting} that maintains the original trigger during the training process but amplifies it during the inference time.  Specifically, we construct the test poisoned sample with the following formula: ${\bm x}_i+\epsilon\cdot\text{sign}({\bm t}_i), \ i=1,2,\cdots,d$, where $\epsilon$ controls the visibility. In our experiments, we set the number of maximum perturbed pixels as $50$. Other training details are consistent with those in Section \ref{sec:main_exp}. We illustrate the poisoned samples in Figure \ref{fig:amp_tri_sample} and depict the ASR curves in Figure \ref{fig:amplify}, from which we find that amplified triggers not only outperform the original triggers but also could be implemented as backdoor patches \cite{brown2017adversarial, liu2020bias} in physical world.  

\vspace{0.3em}
\textcolor{black}{\subsubsection{Effectiveness of SIBA under Clean-label Settings} 
To evaluate the effectiveness of SIBA under clean-label settings, we conduct experiments on the CIFAR-10 dataset. It is known that clean-label attacks are more challenging than poisoned-label ones and we thus set $k=200, \epsilon=16/255$. The results are summarized in Table \ref{tab:cl_siba}, from which we could find that our SIBA achieves $>90\%$ ASR with $5\%$ poisoning rate and outperforms the other baselines with a notable margin. However, we have to acknowledge that when the sparsity is further reduced ($e.g.$, $k=100$), clean-label settings may not achieve high ASRs like poisoned-label ones. We will explore more advanced optimization techniques to push the limit of sparsity under clean-label settings in our future work.}


\begin{figure}[!t]
    \centering
    \includegraphics[width=0.7\linewidth]{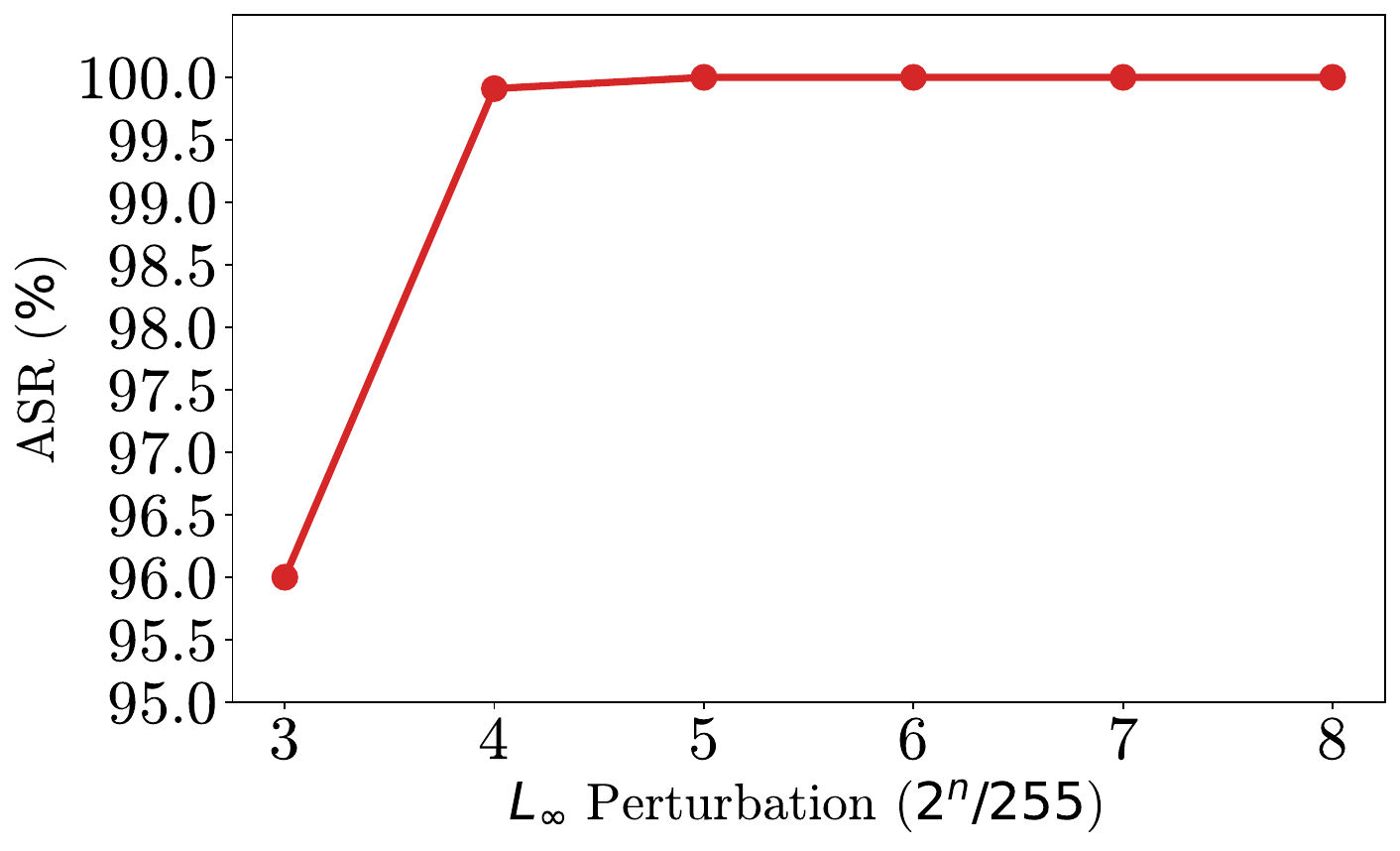}
    \caption{Results of SIBA with amplified triggers on CIFAR-10.}
    \label{fig:amplify}
\end{figure}

\begin{table}[!t]
\centering
\caption{\textcolor{black}{Comparison with the baseline attacks under clean-label settings on the CIFAR-10 dataset.}}
\vspace{-0.3em}
\scalebox{0.83}{
\begin{tabular}{c|c|ccccc }
\toprule
    Method$\downarrow$  & Metric$\downarrow$, Poisoning Rate$\rightarrow$ & $1\%$ & $2\%$  & $3\%$ & $4\%$  & $5\%$  \\
\midrule
 \multirow{2}*{Random} & BA (\%) & 94.51  & 94.42 & 94.56  & 94.10  & 93.87 \\
 & ASR (\%) & 1.51  &  4.91  & 45.58  &  68.89  & 79.81 \\
 \midrule
 \multirow{2}*{Sparse} & BA (\%) & 94.48  & 94.80  & 94.26  & 94.34  & 94.06 \\
 & ASR (\%) & 24.76  &  46.49  &  59.87 &  61.22  &  73.47\\
 \midrule
 \multirow{2}*{SIBA} & BA (\%) & 94.44 & 94.22  & 94.45  & 94.40  & 94.17 \\
 & ASR (\%) & 62.35  &  75.57  &  76.89 &  86.30  & 91.73\\
 \bottomrule
\end{tabular}
}
\label{tab:cl_siba}
\end{table}

\section{Conclusion}
\label{sec:conclusion}
In this paper, we proposed a novel backdoor attack, $i.e.$, sparse and invisible backdoor attack (SIBA), to achieve attack effectiveness and attack stealthiness simultaneously. Our SIBA method only needs to modify a few pixels of the original images to generate poisoned samples and is \textcolor{black}{human-imperceptible} due to the low modification magnitude. To achieve it, we formulated the trigger generation as a bi-level optimization problem with sparsity and invisibility constraints and proposed an effective method to solve it. We conducted extensive experiments on benchmark datasets, verifying the effectiveness, the resistance to potential defenses, and the flexibility under different settings of our attack. We hope our method can provide a new angle and deeper understanding of backdoor mechanisms, to facilitate the design of more secure and robust DNNs.

\section*{Acknowledgments}
The work is supported in part by the National Natural Science Foundation of China (U20A20178, 62171248, U20B2049, and U21B2018), Shenzhen Science and Technology Program (JCYJ20220818101012025), and the PCNL KEY project (PCL2023AS6-1). This work was mostly done when Yiming Li was a Research Professor at Zhejiang University. He is currently a Research Fellow at Nanyang Technological University.


\appendix



\subsection{Proof of Lemma 1}

\setcounter{theorem}{0}
\setcounter{lemma}{0}

\begin{lemma}
Assuming $\alpha=0$ in Equation \ref{eq:fgsm} and the initial value of ${\bm t}_i$ is 0, Problem \ref{eq:square} has the analytical solution as follows:
\begin{equation}
    {\bm t}_{i+1, j} = \left\{
\begin{aligned}
 &{\bm v}_{i, j} &\text{if} \ j \in C^{\prime}\\
&0  &\text{if} \  j \notin C^{\prime}
\end{aligned},
\right.
\end{equation}
where $C^{\prime}$ represents the subscript group which has the largest $k$ element of $\lvert \nabla_{\bm t} h({\bm t}_i) \rvert$.
\end{lemma}
\begin{proof}
We denote $C$ as the $k$-dimension subset of ${\bm v}_i$ such that ${\bm u}_j={\bm v}_{i,j}$ if $j\in C$ and ${\bm u}_j=0$ if $j\notin C$ and assume that the initial value of ${\bm t}_i$ is 0. Then, the objective in the equation (\ref{eq:square}) could be derived as follows:
\begin{equation}
\begin{aligned}
\lVert {\bm s}_i- {\bm u}\rVert_2^2=&\sum_{j\in C} ({\bm s}_{i,j}- {\bm v}_{i,j})^2 +\sum_{j\notin C} {\bm s}_{i,j}^2\\
=&\sum_{j\in C} \left(\alpha\cdot(\nabla_{\bm t} h({\bm t}_i))_j-\epsilon\cdot \text{sign}(\nabla_{\bm t} h({\bm t}_i))_j\right)^2+\\
&\sum_{j\notin C} \left(\alpha\cdot\nabla_{\bm t} h({\bm t}_i)\right)_j^2\\
=&\sum_{j\in C}  \left(\left(\alpha\cdot\lvert(\nabla_{\bm t} h({\bm t}_i))_j\rvert-\epsilon\right)\cdot \text{sign}(\nabla_{\bm t} h({\bm t}_i))_j\right)^2+\\
&\sum_{j\notin C} \left(\alpha\cdot\nabla_{\bm t} h({\bm t}_i)\right)_j^2\\
=&\sum_{j\in C} \left(\alpha\cdot\lvert(\nabla_{\bm t} h({\bm t}_i))_j\rvert-\epsilon\right)^2+\\
&\sum_{j\notin C} (\alpha\cdot\nabla_{\bm t} h({\bm t}_i))_j^2\\
=&\sum_{j\in C}  ((\alpha\cdot\nabla_{\bm t} h({\bm t}_i))_j)^2-\sum_{j\in C}  2\alpha\epsilon\cdot\lvert\nabla_{\bm t} h({\bm t}_i)_j\rvert+\\
&\sum_{j\in C}\epsilon^2+\sum_{j\notin C} ((\alpha\cdot\nabla_{\bm t} h({\bm t}_i))_j)^2\\
=&\lVert\alpha\cdot\nabla_{\bm t} h({\bm t}_i)\rVert_2^2+k\epsilon^2-2\alpha\epsilon\sum_{j\in C}\lvert\nabla_{\bm t} h({\bm t}_i))_j\rvert
\end{aligned}
\label{eq:proof}
\end{equation}
Observing the equation (\ref{eq:proof}), the first two terms are constants and the equation is minimized when the last term contains the largest $k$ element of  $\lvert \nabla_{\bm t} h({\bm t}_i) \rvert$.
\end{proof}

\bibliographystyle{IEEEtran}
\bibliography{tifs.bib}

\begin{IEEEbiography}[{\includegraphics[width=0.9in,height=1.25in,clip,keepaspectratio]{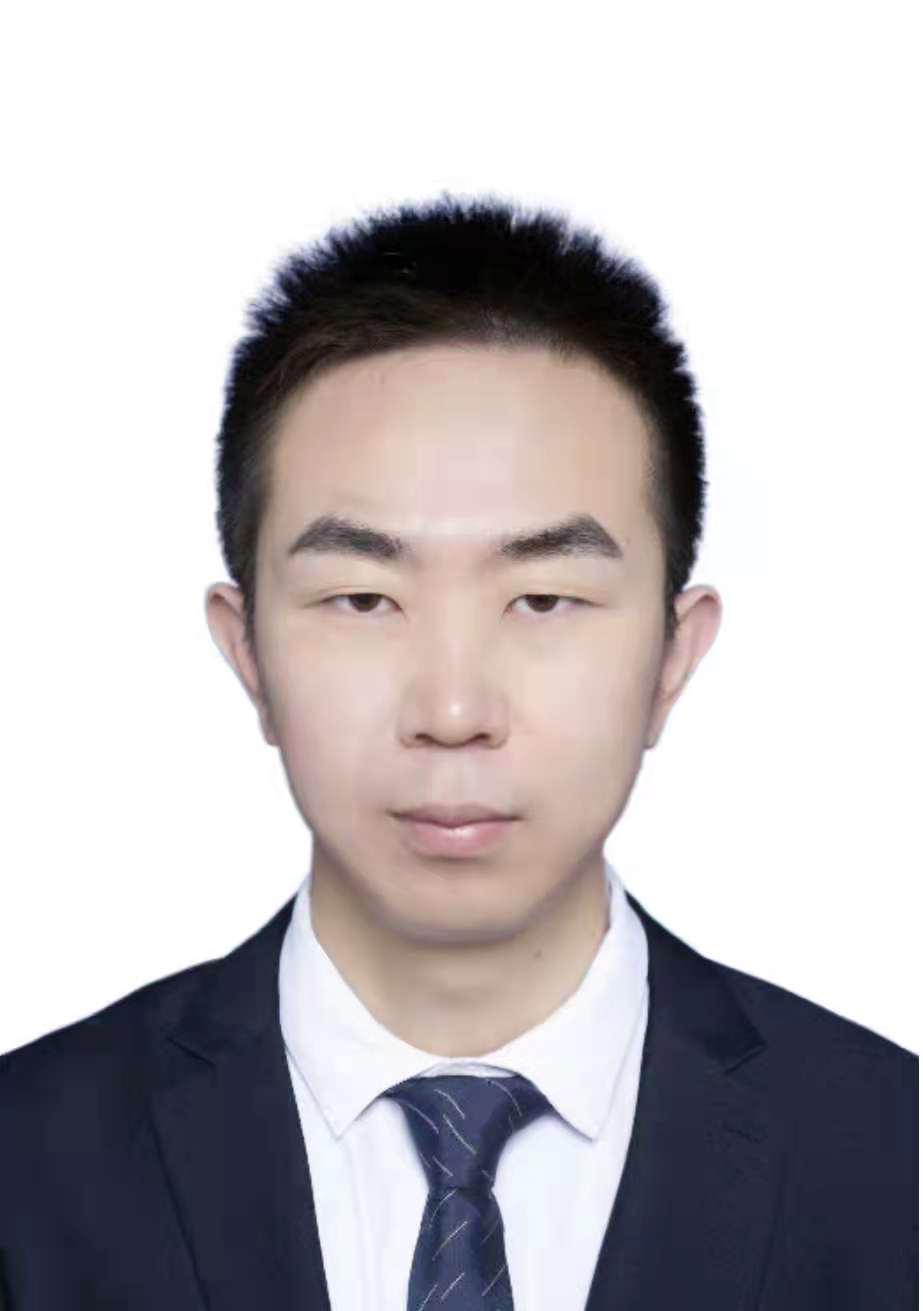}}]{Dr. Yinghua Gao} received his Ph.D. degree from the Department of Computer Science and Technology, Tsinghua University in 2024 and B.S. degree from the Department of Mathematics, Nankai University in 2018. His research interests are primarily in trustworthy machine learning.
\end{IEEEbiography}

\begin{IEEEbiography}[{\includegraphics[width=0.9in,height=1.25in,clip,keepaspectratio]{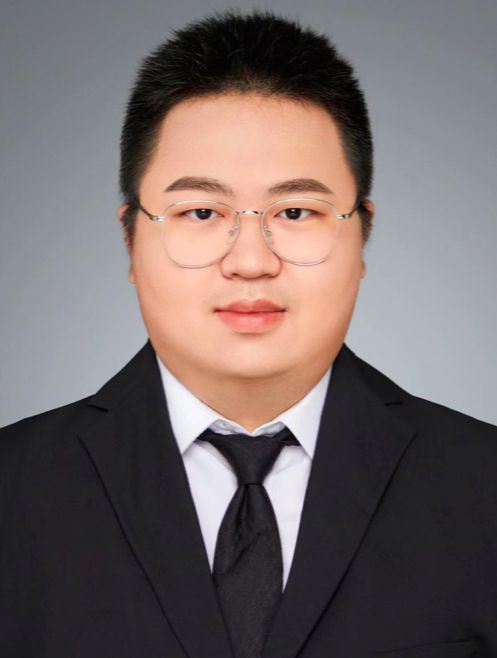}}]{Dr. Yiming Li} is currently a Research Fellow at Nanyang Technological University. Before that, he was a Research Professor in the State Key Laboratory of Blockchain and Data Security at Zhejiang University and also in HIC-ZJU. He received his Ph.D. degree with honors in Computer Science and Technology from Tsinghua University in 2023 and his B.S. degree with honors in Mathematics from Ningbo University in 2018. His research interests are in the domain of Trustworthy ML and Responsible AI, especially backdoor learning and AI copyright protection. His research has been published in multiple top-tier conferences and journals, such as ICLR, NeurIPS, and IEEE TIFS. He served as the Area Chair of ACM MM, the Senior Program Committee Member of AAAI, and the Reviewer of IEEE TPAMI, IEEE TIFS, IEEE TDSC, etc. His research has been featured by major media outlets, such as IEEE Spectrum. He was the recipient of the Best Paper Award at PAKDD in 2023 and the Rising Star Award at WAIC in 2023.
\end{IEEEbiography}

\begin{IEEEbiography}[{\includegraphics[width=1in,height=1.25in,clip,keepaspectratio]{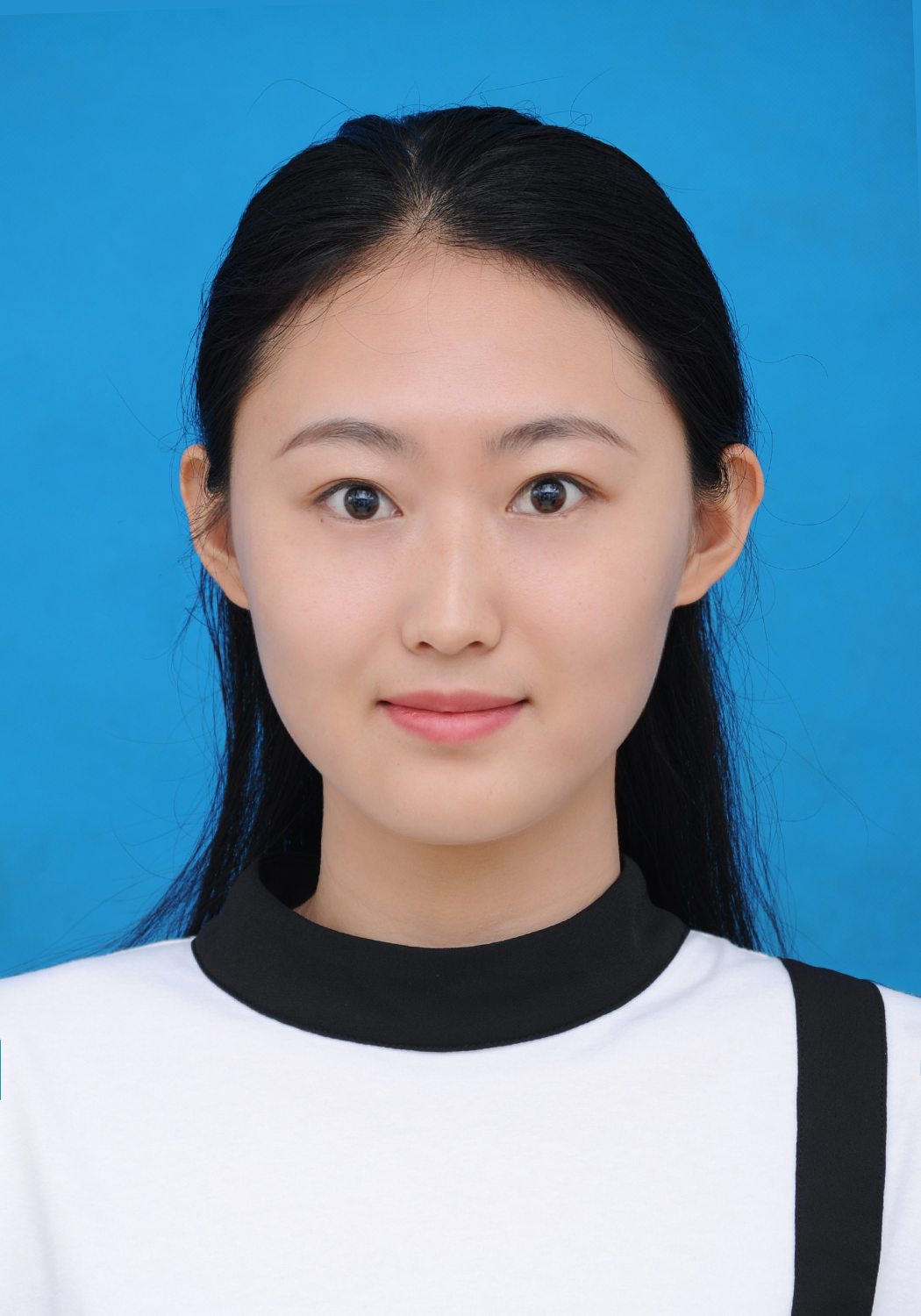}}]{Dr. Xueluan Gong} received her B.S. degree in Computer Science and Electronic Engineering from Hunan University in 2018. She received her Ph.D. degree in Computer Science from Wuhan University in 2023. Her research interests include AI security and information security.
\end{IEEEbiography}

\begin{IEEEbiography}[{\includegraphics[width=1in,height=1.25in,clip,keepaspectratio]{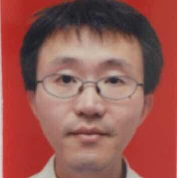}}]{Dr. Zhifeng Li} is currently a top-tier principal research scientist at Tencent Data Platform. He received the Ph.D. degree from the Chinese University of Hong Kong in 2006. After that, He was a postdoctoral fellow at the Chinese University of Hong Kong and Michigan State University for several years. Before joining Tencent, he was a full professor with the Shenzhen Institutes of Advanced Technology, Chinese Academy of Sciences. His research interests include deep learning, computer vision and pattern recognition, and face detection and recognition. He is currently serving on the Editorial Boards of Pattern Recognition, IEEE Transactions on Circuits and Systems for Video Technology, and Neurocomputing. He is a fellow of the British Computer Society (FBCS).
\end{IEEEbiography}

\begin{IEEEbiography}[{\includegraphics[width=0.9in,height=1.25in,clip,keepaspectratio]{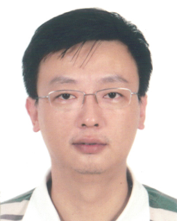}}]{Dr. Shu-Tao Xia} received the B.S. degree in mathematics and the Ph.D. degree in applied mathematics from Nankai University, Tianjin, China, in 1992 and 1997, respectively. Since January 2004, he has been with the Tsinghua Shenzhen International Graduate School of Tsinghua University, Guangdong, China, where he is currently a full professor. From September 1997 to March 1998 and from August to September 1998, he visited the Department of Information Engineering, The Chinese University of Hong Kong, Hong Kong. His research interests include coding and information theory, machine learning, and deep learning. His papers have been published in multiple top-tier journals and conferences, such as IEEE TPAMI, IEEE TIFS, IEEE TDSC, CVPR, ICLR, NeurIPS.
\end{IEEEbiography}

\begin{IEEEbiography}[{\includegraphics[width=1in,height=1.25in, clip,keepaspectratio]{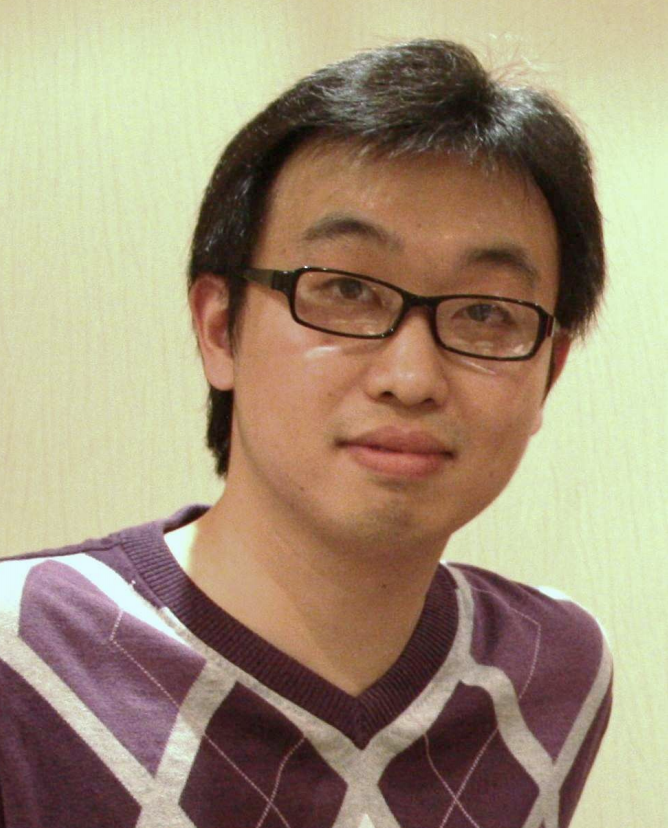}}]{Dr. Qian Wang} (Fellow, IEEE) is a Professor in the School of Cyber Science and Engineering at Wuhan University, China. He was selected into the National High-level Young Talents Program of China, and listed among the World's Top 2\% Scientists by Stanford University. He also received the National Science Fund for Excellent Young Scholars of China in 2018. He has long been engaged in the research of cyberspace security, with focus on AI security, data outsourcing security and privacy, wireless systems security, and applied cryptography. He was a recipient of the 2018 IEEE TCSC Award for Excellence in Scalable Computing (early career researcher) and the 2016 IEEE ComSoc Asia-Pacific Outstanding Young Researcher Award. He has published 200+ papers, with 120+ publications in top-tier international conferences, including USENIX NSDI, ACM CCS, USENIX Security, NDSS, ACM MobiCom, ICML, etc., with 20000+ Google Scholar citations. He is also a co-recipient of 8 Best Paper and Best Student Paper Awards from prestigious conferences, including ICDCS, IEEE ICNP, etc. In 2021, his PhD student was selected under Huawei's  `Top Minds' Recruitment Program. He serves as Associate Editors for IEEE Transactions on Dependable and Secure Computing (TDSC) and IEEE Transactions on Information Forensics and Security (TIFS). 
\end{IEEEbiography}

\end{document}